\documentclass[11pt]{article}
\usepackage{amsmath}
\usepackage{amsthm}
\allowdisplaybreaks
\usepackage{amssymb}
\usepackage{algorithm}
\usepackage[lined,boxed,ruled,norelsize,algo2e,linesnumbered]{algorithm2e}
\usepackage{subfig}

\usepackage{array}
\usepackage{color}
\usepackage[dvipsnames]{xcolor}
\usepackage{graphicx}
\usepackage{wrapfig,epsfig}
\usepackage{epstopdf}
\usepackage{url}
\usepackage{graphicx}
\usepackage{color}
\usepackage{epstopdf}
\usepackage{algpseudocode}
\usepackage{scrextend}
\usepackage[T1]{fontenc}
\usepackage{bbm}
\usepackage{comment}
\usepackage{authblk}
\usepackage{multicol}
\usepackage{dsfont}
\usepackage{mathtools}
\usepackage{enumitem}
\usepackage{mathrsfs}

\usepackage{tikz}
\usepackage{hyperref}  
\hypersetup{colorlinks=true,citecolor=blue,linkcolor=red}

\usetikzlibrary{arrows}
\usepackage[margin=1in]{geometry}

\graphicspath{{./figs/}}

\definecolor{b2}{RGB}{51,153,255}
\definecolor{mygreen}{RGB}{80,180,0}

\newcommand{\tabincell}[2]{\begin{tabular}{@{}#1@{}}#2\end{tabular}}

\theoremstyle{plain}
\newtheorem{theorem}{Theorem}[section]
\newtheorem{lemma}[theorem]{Lemma}

\newtheorem{fact}[theorem]{Fact}
\newtheorem{claim}[theorem]{Claim}
\newtheorem{corollary}[theorem]{Corollary}

\theoremstyle{definition}
\newtheorem{definition}[theorem]{Definition}

\theoremstyle{remark}
\newtheorem{remark}[theorem]{Remark}

\newcommand{\wh}{\widehat}

\renewcommand{\epsilon}{\varepsilon}
\renewcommand{\phi}{\varphi}

\newcommand{\A}{\mathcal{A}}

\newcommand{\I}{\mathbf{I}}
\newcommand{\N}{\mathcal{N}}
\newcommand{\R}{\mathbb{R}}

\newcommand{\HF}{\hat{F}}
\newcommand{\Fnr}{\hat{F}_{n_r}}
\newcommand{\K}{\mathcal{K}}

\renewcommand{\hat}{\wh}

\renewcommand{\d}{\mathrm{d}}

\newcommand{\ACSA}{\ensuremath{\mathsf{AC-}}\ensuremath{\mathsf{SA}}\xspace}

\newcommand{\METADP}{\ensuremath{\mathsf{META_{DP}}}\xspace}
\newcommand{\aux}{\ensuremath{\mathbf{aux}}}
\newcommand{\homega}{\hat{\omega}}

\DeclareMathOperator*{\E}{\mathbb{E}}
\DeclareMathOperator*{\M}{\mathcal{M}}

\DeclareMathAlphabet{\mathpzc}{OT1}{pzc}{m}{it}

\title{Private Non-smooth Empirical Risk Minimization and Stochastic Convex Optimization in Subquadratic Steps}

\author{
Janardhan Kulkarni\thanks{Algorithms Group, Microsoft Research (MSR) Redmond. \texttt{jakul@microsoft.com}}
\quad
Yin Tat Lee\thanks{\texttt{yintat@uw.edu}. University of Washington and Microsoft Research. Supported by NSF awards CCF-1749609, DMS-1839116, DMS-2023166, Microsoft Research Faculty Fellowship, Sloan Research Fellowship, Packard Fellowships.}
\quad 
Daogao Liu\thanks{\texttt{liudaogao@gmail.com}. University of Washington. Part of the work was done while visiting Shanghai Qi Zhi Institute.}
}

\date{}

\graphicspath{{./figs/}}

\begin{document}

\begin{titlepage}
  \maketitle
  \begin{abstract}
We study the differentially private Empirical Risk Minimization (ERM) and Stochastic Convex Optimization (SCO) problems for non-smooth convex functions. 
We get a (nearly) optimal bound on the excess empirical risk and excess population loss with subquadratic gradient complexity.
More precisely, our differentially private algorithm requires $O(\frac{N^{3/2}}{d^{1/8}}+ \frac{N^2}{d})$ gradient queries for optimal excess empirical risk, which is achieved with the help of subsampling and smoothing the function via convolution. 
This is the first subquadratic algorithm for the non-smooth case when $d$ is super constant.
As a direct application, using the iterative localization approach of Feldman et al. \cite{fkt20}, we achieve the optimal excess population loss for stochastic convex optimization problem, with $O(\min\{N^{5/4}d^{1/8},\frac{ N^{3/2}}{d^{1/8}}\})$ gradient queries.
Our work makes progress towards resolving a question raised by Bassily et al. \cite{bfgt20}, giving first algorithms for private ERM and SCO with subquadratic steps. 

We note that independently Asi et al. \cite{afkt21} gave other algorithms for private ERM and SCO with subquadratic steps.

  \end{abstract}

  \thispagestyle{empty}
\end{titlepage}


{\hypersetup{linkcolor=black}
\tableofcontents
}
\newpage


\section{Introduction}
Privacy has become an important consideration for learning algorithms dealing  with sensitive data. 
Over the past decade, differential privacy, introduced in the seminal work of \cite{dmns06}, has established itself as the defacto notion of privacy for machine learning problems. 
In this paper, we revisit Empirical Risk Minimization (ERM) and Stochastic Convex Optimization (SCO) problem, which are one of the most important and simplest problems in statistics and machine learning, in  differential privacy setting. 
In the ERM problem, we are given a family of convex functions $\{f(\cdot,x)\}_{x\in \Xi}$ over a closed convex set $\K\subset \R^d$, a data set $S=\{x_1,\cdots,x_N\}$ drawn from some unknown distribution ${\cal P}$ over the universe $\Xi$, and the objective is to
\begin{align*}
    \text{minimize}\quad \HF(\omega):=\frac{1}{N}\sum_{x_i\in S}f(\omega,x_i) \quad\text{over}\quad \omega\in \K,
\end{align*}
while in the SCO the objective is to
\begin{align*}
    \text{minimize}\quad F(\omega):=\E_{x\sim {\cal P}}f(\omega,x) \quad\text{over}\quad \omega\in \K,
\end{align*}

Differentially private 
convex optimization has been studied extensively for over a decade now \cite{cm08,rbht09,cms11,kst12,jt14,ttz14,bst14,ttz15,kj16,wlk+17,fts17,zzmw17,wyx17,ins+19}. 
Most of the previous results are focus on DP-ERM and roughly speaking, there are three major approaches in DP-ERM: output perturbation, objective perturbation, and gradient perturbation. 
Output perturbation approach is based on the sensitivity method proposed by  \cite{dmns06} and adds noise to the final output to the standard ERM problem \cite{cm08,rbht09,cms11,zzmw17}. 
Objective perturbation \cite{cm08,cms11,kst12,ttz14} means to perturb the objective function we want to minimize. 
In the gradient perturbation approach, we add noise to the first order information using optimization algorithms such as Stochastic Gradient Descent (SGD). 
This approach was first proposed in \cite{bst14} and was later extended by \cite{ttz14,wyx17}, and has lead to the state-of-the-art theoretical bounds for DP-ERM.
For an experimental comparison of various approaches to solving DP-ERM we refer the readers to \cite{rbht09,ins+19}.




DP-ERM for smooth convex functions is well understood in the sense that we know (near) linear time algorithms that achieve optimal excess empirical risk. 
We refer the readers to \cite{wyx17} for more details. 
However, for the more general non-smooth convex loss functions our understanding is not yet complete, which is the focus of this paper.
A summary of the state-of-the-art results and our contributions for the non-smooth convex loss functions is given in Table~\ref{tab:general_convex} (General Convex) and Table~\ref{tab:strongly} (Strongly Convex).
We will discuss the concurrent work \cite{afkt21} separately at the end of the introduction, and the following discussion are only limited to the previous work.

\cite{kst12} used the objective perturbation method to design a DP-algorithm with $O(\frac{GD\sqrt{d}\log(1/\delta)}{\sqrt{N}\epsilon})$ excess empirical risk.
This result was improved significantly by \cite{bst14}, who first showed a lower bound of $\Omega(\min\{GD,\frac{GD\sqrt{d}}{N\epsilon}\})$ on the excess empirical risk for DP-ERM. 
Further, they gave an algorithm with excess empirical risk $O(\frac{GD \log ^{\frac{3}{2}}(N / \delta) \sqrt{d \log (1 / \delta)}}{N\epsilon})$, which is sub-optimal by a factor of  $\log ^{\frac{3}{2}}(N / \delta)$.
Their algorithm is based on a modification of SGD by adding Gaussian noise to the gradients to make it DP.
The privacy analysis proceeds via amplification by sampling and the strong composition theorem.
Roughly speaking, the logarithmic blowup in the excess empirical risk is due to two reasons: 1) the strong composition theorem requires that at each step one needs to add Gaussian noise with a larger variance;
2) They used sub-optimal convergence rate $O(\log T/\sqrt{T})$ for $T$-step SGD. 

However, getting the optimal bounds with small gradient complexity for non-smooth case turns out to be a more difficult problem.
This was noted by \cite{wyx17}, who raised it as an important open problem. 
This question was answered in \cite{bftt19}, who gave an algorithm with almost optimal excess empirical risk.
 To achieve this, \cite{bftt19} first consider the smooth case, and give an improved privacy analysis via the Moments Accountant technique proposed by \cite{acg+16}. 
They extend their result to non-smooth case by applying Moreau-Yosida envelope technique (a.k.a. Moreau envelope smoothing) \cite{nes05} to make the function smooth.
However, this technique is  computationally inefficient and leads to $O(N^{4.5})$-gradient computations for the whole algorithm.
This limitation was overcome in a recent work of \cite{bfgt20} who gave the optimal excess empirical risk guarantee with $O(N^2)$-gradient computations. 
The privacy analysis of this result also used Moments Accountant method, and they used the standard online-to-batch conversion technique \cite{ccg04} to prove the high-probability bound on the excess empirical error of SGD, which leads to the near optimal bound in expectation.
We remark that all the papers \cite{bftt19,bfgt20} above not only study the ERM problem, but also consider more general  DP-SCO settings and uniform stability, and in some cases, results on ERM are byproducts of the more general results. 

As we can see from Table~\ref{tab:general_convex} and Table~\ref{tab:strongly}, all the previously known results (except the concurrent work \cite{afkt21}) achieving near optimal excess empirical risk bounds require at least $O(N^2)$-gradient computations.
It is natural to ask if there are lower bounds to rule out algorithms with subquadratic gradient complexity that can match the error bounds of the above results.

\begin{center}
\begin{table}[h]
    \centering
    \begin{tabular}{ | m{5em} | m{ 3.5cm}| m{1.7cm} |}
\hline
 & Excess Empirical Risk & Gradient Complexity \\ 
\hline 
 \cite{kst12} & $\frac{GD\sqrt{d}\log(1/\delta)}{\sqrt{N}\epsilon} $ & N/A\\
\hline
\cite{bst14} & $\frac{GD \log ^{\frac{3}{2}}(N / \delta) \sqrt{d \log (1 / \delta)}}{N\epsilon}$ & $N^2$ \\ 
\hline
\cite{bftt19} & $\frac{GD\sqrt{d\log(1/\delta)}}{N\epsilon}$ & $N^{4.5}$ \\ 
\hline
\cite{bfgt20} & $\frac{GD\sqrt{d\log(1/\delta)}}{N\epsilon}$  &  $N^{2}$\\
\hline
\cite{afkt21} & $\frac{GD\sqrt{d\log(1/\delta)}}{N\epsilon}$ & $N^2/\sqrt{d}$ \\ 
\hline
Ours & $\frac{GD\sqrt{d\log(1/\delta)}}{N\epsilon}$ & \tabincell{c}{$\frac{N^{3/2}}{d^{1/8}}+  \frac{N^2}{d}$}\\
\hline
\end{tabular}
    \caption{Comparisons with previous $(\epsilon,\delta)$-differential private algorithms when objective function is $G$-Lipschitz and convex over a convex set $\K\subset \R^d$ of diameter $D$. The results are stated asymptotically and the big $O$ notation is hidden for simplicity. The lower bound is $\Omega(\min\{GD,\frac{GD\sqrt{d}}{N\epsilon}\})$ \cite{bst14}.}
    \label{tab:general_convex}
\end{table}

\begin{table}[h!]
    \centering
    \begin{tabular}{ | m{5em} | m{ 3.5 cm}| m{1.7cm} |}
\hline
 & Excess Empirical Risk & Gradient Complexity \\
 \hline
 \cite{kst12} & $\frac{G^2d\log(1/\delta)}{\mu  N^{3/2} \epsilon^2}$ & N/A \\
\hline
\cite{bst14} & $\frac{G^2 \log ^2(N / \delta) d \log (1 / \delta)}{\mu N^2\epsilon^2}$ & $N^2$ \\ 
\hline
\cite{bftt19} & $\frac{G^2d\log(1/\delta)}{\mu N^2\epsilon^2}$ & $N^{4.5}$ \\ 
\hline
\cite{bfgt20} & $\frac{G^2d\log(1/\delta)}{\mu N^2\epsilon^2}$ &  $N^{2}$\\
\hline
Ours & $\frac{G^2d\log(1/\delta)}{\mu N^2\epsilon^2}$ & \tabincell{c}{$ \frac{N^{3/2}}{d^{1/8}}+\frac{N^2}{d}$} \\
\hline
\end{tabular}
    \caption{Comparisons with previous $(\epsilon,\delta)$-differential private algorithms when objective function is $G$-Lipschitz and $\mu$-strongly convex over a convex set $\K\subset \R^d$. The results are stated asymptotically and the big $O$ notation is hidden for simplicity. The lower bound is $\Omega(\min\{\frac{G^2}{\mu},\frac{G^2d}{\mu N^2\epsilon^2}\})$ \cite{bst14}.}
    \label{tab:strongly}
\end{table}
\end{center}
As Table~\ref{tab:general_SCO} and Table~\ref{tab:strongly_SCO} show, a similar situation arises in Stochastic Convex Optimization (SCO), which is a closely related problem compared to ERM.
In the SCO problem, we want to minimize the objective function $F(\omega)=\E_{x \sim {\cal P}}[f(\omega,x)]$  for some unknown distribution ${\cal P}$ over the universe $\Xi$. 
Many results for SCO  \cite{bst14,bftt19,bfgt20} are directly based on ERM; that is, solving the ERM and analyzing the generalization error. 
The first non-trivial result for general convex loss functions achieving excess population loss of $O\left( GD(\frac{d^{1/4}}{\sqrt{N}}+\frac{\sqrt{d}}{N\epsilon})\right)$ was given by \cite{bst14}, who showed the result by first solving the ERM problem and bounding the generalization error.
They used the result on universal convergence directly, namely, bounding $\sup_{\omega \in \K}\E[F(\omega)-\HF(\omega)]$.
But this method has its limitations; For example, \cite{fel16} showed that lower bound of universal convergence is $\Omega(\sqrt{d/N})$ for some (not necessarily convex) loss functions. 
Later, \cite{bftt19}, \cite{fkt20} and \cite{bfgt20} obtained near optimal excess population loss with significantly better running times (gradient complexity). 
The privacy analysis in these papers relied on recent advances in the privacy techniques such as the Moments Accountant method \cite{acg+16}, Rényi differential privacy (RDP) \cite{mir18} and the Privacy Amplification by Iteration \cite{fmtt18} and other fast stochastic convex optimization algorithms such as \cite{jnn19}. 
The excess population loss bound in most of these works followed by solving a (phased) convex (regularized) ERM problem and then appealing to the uniform stability property \cite{hrs16} or the iterative localization approach \cite{fkt20} to do the generalization error analysis.

\begin{center}
    \begin{table}[h!]
    \centering
    \begin{tabular}{ | m{5em} | m{ 7 cm}| m{5cm} |}
\hline
 & Excess Population Loss & Gradient Complexity \\
\hline
\cite{bst14} & $GD(\frac{d^{1/4}\log(n/\delta)}{\sqrt{N}}+\frac{d^{1/2}\log^2(n/\delta))}{N\epsilon})$ & $N^2$ \\ 
\hline
\cite{bftt19} &$ GD(\frac{1}{\sqrt{N}}+\frac{\sqrt{d\log(1/\delta)}}{N\epsilon})$& $N^{4.5}$ \\ 
\hline
\cite{fkt20} & $ GD(\frac{1}{\sqrt{N}}+\frac{\sqrt{d\log(1/\delta)}}{N\epsilon})$ & $N^2\log(1/\delta)$ \\
\hline
\cite{bfgt20} & $ GD(\frac{1}{\sqrt{N}}+\frac{\sqrt{d\log(1/\delta)}}{N\epsilon})$ &  $N^{2}$\\
\hline
\cite{afkt21} & $ GD(\frac{1}{\sqrt{N}}+\frac{\sqrt{d\log(1/\delta)}}{N\epsilon})$ & $\min\{N^{3/2},N^2/\sqrt{d}\}$ \\
\hline
Ours & $GD(\frac{1}{\sqrt{N}}+\frac{\sqrt{d\log(1/\delta)}}{N\epsilon})$ & $\min\{N^{5/4}d^{1/8},N^{3/2}/d^{1/8}\}$ \\
\hline
\end{tabular}
    \caption{Comparisons with previous $(\epsilon,\delta)$-differential private algorithms when objective function is $G$-Lipschitz and convex over a convex set $\K\subset \R^d$. The results are stated asymptotically and the big $O$ notation is hidden for simplicity. The lower bound is $\Omega(GD(\frac{1}{\sqrt{N}}+\frac{\sqrt{d}}{N\epsilon}))$ \cite{bst14}.}
    \label{tab:general_SCO}
\end{table}

\begin{table}[h!]
    \centering
    \begin{tabular}{ | m{5em} | m{ 7 cm}| m{5 cm} |}
\hline
 & Excess Population Loss & Gradient Complexity \\
\hline
\cite{bftt19} & $\frac{G^2}{\mu}(\frac{1}{N}+\frac{d\log(1/\delta)}{N^2\epsilon^2})$ & $N^{4.5}$ \\ 
\hline
\cite{fkt20} & $\frac{G^2}{\mu}(\frac{1}{N}+\frac{d\log(1/\delta)}{N^2\epsilon^2})$ & $N^{2}\log(1/\delta)$ \\
\hline
\cite{bfgt20} & $\frac{G^2}{\mu}(\frac{1}{N}+\frac{d\log(1/\delta)}{N^2\epsilon^2})$ &  $N^{2}$\\
\hline
\cite{afkt21} & $\frac{G^2}{\mu}(\frac{1}{N}+\frac{d\log(1/\delta)}{N^2\epsilon^2})$ & $\min\{N^{3/2},N^2/\sqrt{d}\}$ \\
\hline
Ours & $\frac{G^2}{\mu}(\frac{1}{N}+\frac{d\log(1/\delta)}{N^2\epsilon^2})$ & $\min\{N^{5/4}d^{1/8},N^{3/2}/d^{1/8}\}$ \\
\hline
\end{tabular}
    \caption{Comparisons with previous $(\epsilon,\delta)$-differential private algorithms when objective function is $G$-Lipschitz and $\mu$-strongly convex over a convex set $\K\subset \R^d$. The results are stated asymptotically and the big $O$ notation is hidden for simplicity. The lower bound is $\Omega(\frac{G^2}{\mu}(\frac{1}{N}+\frac{d\log(1/\delta)}{N^2\epsilon^2}))$ \cite{bst14}.}
    \label{tab:strongly_SCO}
\end{table}
\end{center}

Despite these impressive improvements, as the Table~\ref{tab:general_SCO} and Table~\ref{tab:strongly_SCO} suggest, the previous algorithms (except the concurrent work \cite{afkt21}) which achieve the optimal excess population loss still require $O(N^2)$-gradient computations.
Indeed, \cite{bfgt20} write that \begin{quote}
    {\em `` Proving that quadratic running time is necessary for general non-smooth DP-SCO is a very interesting
    open problem...
    ''}\end{quote}
Understanding if the lower bound is the right answer to the above questions or one can design algorithms with subquadratic gradient complexity is the main motivation that spurred our work.

\subsection{Our Contributions}
Given the close connections between the ERM and SCO problems and the bottleneck on gradient complexity of all known algorithms, it is natural to ask if the open question raised in \cite{bfgt20} also holds for the ERM problem. 
As noted earlier, the state-of-art algorithms for DP-ERM achieving optimal excess empirical risk bounds require $O(N^2)$-gradient computations.

The main contribution of this paper is to show that we can obtain subquadratic gradient complexity bound for ERM when the dimension is super constant.
In particular, for the important regime of over-parameterization ($d \geq N$), we achieve a bound of $N^{1+3/8}$.
Combining our private ERM algorithm and the iterative localization approach proposed in \cite{fkt20}, we can achieve optimal excess population loss with gradient complexity
$O\big( N+
    \min\{\sqrt{\epsilon}N^{5/4}d^{1/8},\frac{\epsilon N^{\frac{3}{2}}}{d^{1/8}\log^{1/4}(1/\delta)}\} 
    \big)$.
    
Let $\K_r=\{y\mid y=\omega+z,\omega\in \K,z\in \R^d,\|z\|_r\leq r \}$.
We now state the main technical contributions of this paper formally.

\begin{theorem}[DP-ERM]
\label{thm:DP-ERM}
Suppose ${\cal K}\subset \R^d$ is a closed convex set of diameter $D$ and $\{f(\cdot,x)\}_{x\in \Xi }$ is a family of $G$-Lipschitz and convex functions over $\K_r$, where $r=\frac{D\sqrt{d\log(1/\delta)}}{\epsilon N}$
\footnote{We only need consider the non-trivial case when $\frac{\sqrt{d\log(1/\delta)}}{\epsilon N}\leq 1$, or any feasible solution is good enough. This means that $r=O(D)$, which is a mild assumption.}. 
For $\epsilon,\delta\leq 1/2$, given any sample set $S$ consists of $N$ samples from $\Xi$ and arbitrary initial point $\omega_0\in \K$, we have a $(\epsilon,\delta)$-differentially private algorithm $\A$ which takes 
$$O\left(\frac{\epsilon N^{\frac{3}{2}}}{d^{1/8}\log^{1/4}(1/\delta)}+\frac{\epsilon^2 N^2}{d\log(1/\delta)}\right)$$ 
gradient queries and outputs $\omega_T$ such that
\begin{align*}
    \E[\HF(\omega_T)-\HF^*]=O\left(\frac{GD\sqrt{d\log(1/\delta)}}{\epsilon N}\right),
\end{align*}
where $D=\|\omega^*-\omega_0\|_2, \HF(\omega)=\frac{1}{N}\sum_{x_i\in S}f(\omega,x_i), \HF^*=\min_{\omega\in \K}\HF(\omega)$, and the expectation is taken over the randomness of the algorithm. 

Moreover, if $\{f(\cdot,x)\}_{x\in \Xi }$ is also $\mu$-strongly convex functions over $\K_r$, we have an $(\epsilon,\delta)$-differentially private algorithm which takes the same bound of gradient queries and outputs $\omega_T$ such that
\begin{align*}
    \E[\HF(\omega_T)-\HF^*]=O\left(\frac{G^2d\log(1/\delta)}{\mu\epsilon^2N^2}\right).
\end{align*}
\end{theorem}

As we have mentioned, combining our private ERM algorithm with the iterative localization technique, we can also give the first algorithm achieving optimal excess population loss with (strictly) sub-quadratic steps for all dimensions:
\begin{theorem}[DP-SCO]
\label{thm:DP-SCO}
Suppose $\epsilon,\delta\leq 1/2$ and sample set $S$ consists of $N$ samples drawn i.i.d from a distribution ${\cal P}$ over $\Xi$.
Let $\{f(\cdot,x)\}_{x\in \Xi}$ is convex and $G$-Lipschitz with respect to $\ell_2$ norm and convex over $\K_r$, where $r=\frac{D\sqrt{d\log(1/\delta)}}{\epsilon N}$, there is an $(\epsilon,\delta)$-differentially private algorithm which takes
\begin{align*}
    O(N+\min\{\sqrt{\epsilon}N^{5/4}d^{1/8},\frac{\epsilon N^{3/2}}{d^{1/8}\log^{1/4}(1/\delta)}\})
\end{align*}
gradient queries to get a solution $\omega_T$
\begin{align*}
    \E[F(\omega_T)-F(\omega^*)]=O(GD(\frac{1}{\sqrt{N}}+\frac{\sqrt{d\log(1/\delta)}}{N\epsilon}).
\end{align*}
Moreover, if $\{f(\cdot,x)\}_{x\in \Xi}$ is also $\mu$-strongly convex over $\K_r$, we can meet the same gradient query complexity and get a solution $\omega_T$ such that:
\begin{align*}
    \E[F(\omega_T)-F(\omega^*)]=O\left(\frac{G^2}{\mu}(\frac{d\log(1/\delta)}{\epsilon^2N^2}+\frac{1}{N})\right).
\end{align*}
\end{theorem}

Finally, we note that our results can also capture the regularized ERM and SCO, which shows up often in the previous work such as \cite{rbht09,kst12,wyx17,ins+19}.
Briefly, in the regularized problem, there is one more simple (and convex) function $h(\omega)$ added to the objective function to encourage certain solutions with better structure. The objective function then takes the form $\frac{1}{N}\sum_{x_i\in S}f(\omega,x_i)+h(\omega)$.
We get asymptotically same results for the regularized ERM/SCO problem with straightforward modifications. 

\subsection{Our Techniques}
Most of the previous works \cite{bst14,bftt19,bfgt20} that achieve near optimal bounds for ERM and SCO are based on adaptations of SGD to make it differentially private.
The information theoretic lower bound of $\Omega(1/\sqrt{T})$ for $T$-step SGD may be one of the important reasons why we can not get subquadratic gradient complexity for non-smooth convex ERM easily. 
Consider the algorithm in \cite{bfgt20} as an example. 
It needs to add Gaussian noise $v\sim\N(0,\sigma^2I_{d\times d})$ with $\sigma^2=\frac{G^2\log(1/\delta)}{\epsilon^2}$ to each gradient. 
By a standard analysis of SGD, we can only show an excess empirical risk of $\Theta(\frac{D\sqrt{d\sigma^2}}{\sqrt{T}})$, which requires us to set $T=\Omega(N^2)$ to get ideal bound, thus hitting the quadratic barrier.

We deviate from the above approaches for designing private algorithms for non-smooth functions.
First notice that the gradient complexity $O(\frac{\epsilon N^{\frac{3}{2}}}{d^{1/8}\log^{1/4}(1/\delta)}+\frac{\epsilon^2 N^2}{d\log(1/\delta)})$ in Theorem~\ref{thm:DP-ERM} is the same for both strongly convex and general non-smooth functions; same holds for DP-SCO.
This is not a coincidence; If we can achieve optimal empirical risk (population loss) for one case, then we can achieve optimal empirical risk (population loss) for another with the same privacy guarantee and gradient complexity.
In fact, the Figure~\ref{fig:reductions} shows the relationship among these different problems.
\begin{figure}[ht]
    \centering
    \includegraphics[width = .5\textwidth]{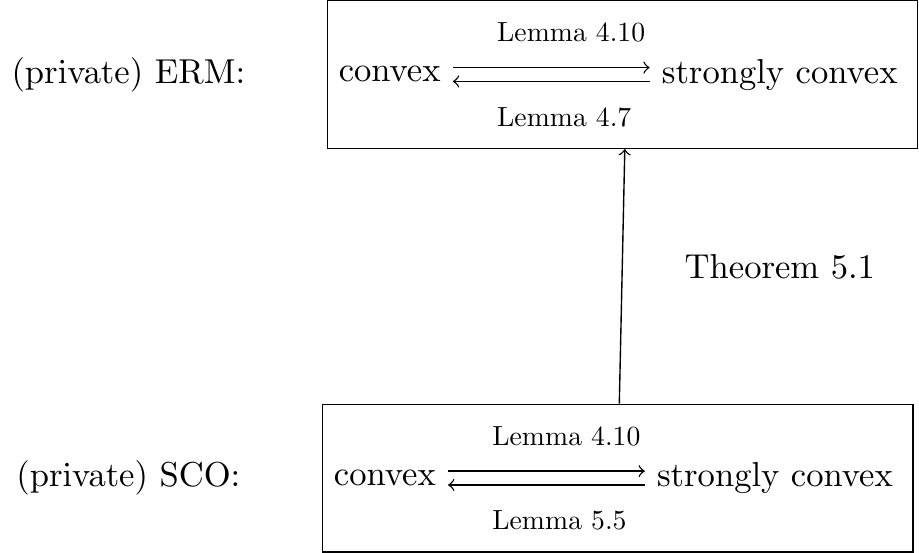}
    \caption{Reductions between ERM and SCO for general convex and strongly convex cases.
    As the lower bound of excess population loss is $\Omega(GD(\frac{1}{\sqrt{N}}+\frac{\sqrt{d}}{N\epsilon}))$ while the lower bound of empirical risk is $\Omega(\frac{GD\sqrt{d}}{N\epsilon})$, we do not know how to reduce from ERM to SCO. 
    }
    \label{fig:reductions}
\end{figure}

Our result for the general convex non-smooth case is obtained by providing a reduction to the strongly convex non-smooth case.
Thus, our task becomes designing better algorithms for the strongly convex non-smooth functions.
Rather than using SGD, we let the objective function take convolution with a sphere kernel to make it smooth. 
We then use the accelerated stochastic approximation algorithm in \cite{gl12} for solving strongly convex stochastic optimization problems.
However, this is not enough, as the required noise that needs to be added to the gradients to make the algorithm private is too large to get subquadratic gradient complexity, even if we use the tighter Moments Accountant technique \cite{acg+16}. 
We overcome this by increasing the batch size to an appropriate value.
Combining these ideas together, we show that the amount of noise we add can be reduced to achieve the optimal excess empirical loss, and we get the gradient complexity of $O(\max\{N^{3/2}/d^{1/8}, N^2/d \})$. 

For SCO, we get the gradient complexity of $O(\min\{N^{5/4} d^{1/8}, N^{3/2}/d^{1/8} \})$ via a direct application of the iterative localization approach of Feldman et al \cite{bfgt20}.
The intuition behind iterative localization is using private ERM to solve regularized objective functions which have low sensitivity, iteration by iteration.
Each iteration reduces the distance to an approximate minimizer by a multiplicative factor, so after logarithmic number of phases we are done.

\subsection{Concurrent and Independent Work}
In an independent and concurrent work, \cite{afkt21} give a new analysis of private regularized mirror descent to do the private ERM.
Then they combine the iterative localization approach to achieve the optimal excess population loss for SCO.
Their result also achieves subquaratic gradient complexity.
More formally, they get $O\left(\log N \cdot \min \left(N^{3 / 2} \sqrt{\log d}, N^{2} / \sqrt{d}\right)\right)$ for SCO in query complexity.
We compare their gradient complexity with ours in Figure~\ref{fig:compare_result}.
Finally, we remark that the main motivation of \cite{afkt21} was to study SCO problem in more general $\ell_p$ norms as much of the literature has focussed on the $\ell_2$-norm.
They also give news results in $\ell_p$-bounded domain together with another concurrent work \cite{bgn21}.

\begin{figure}[h]
    \centering
    \includegraphics[width = .6 \textwidth]{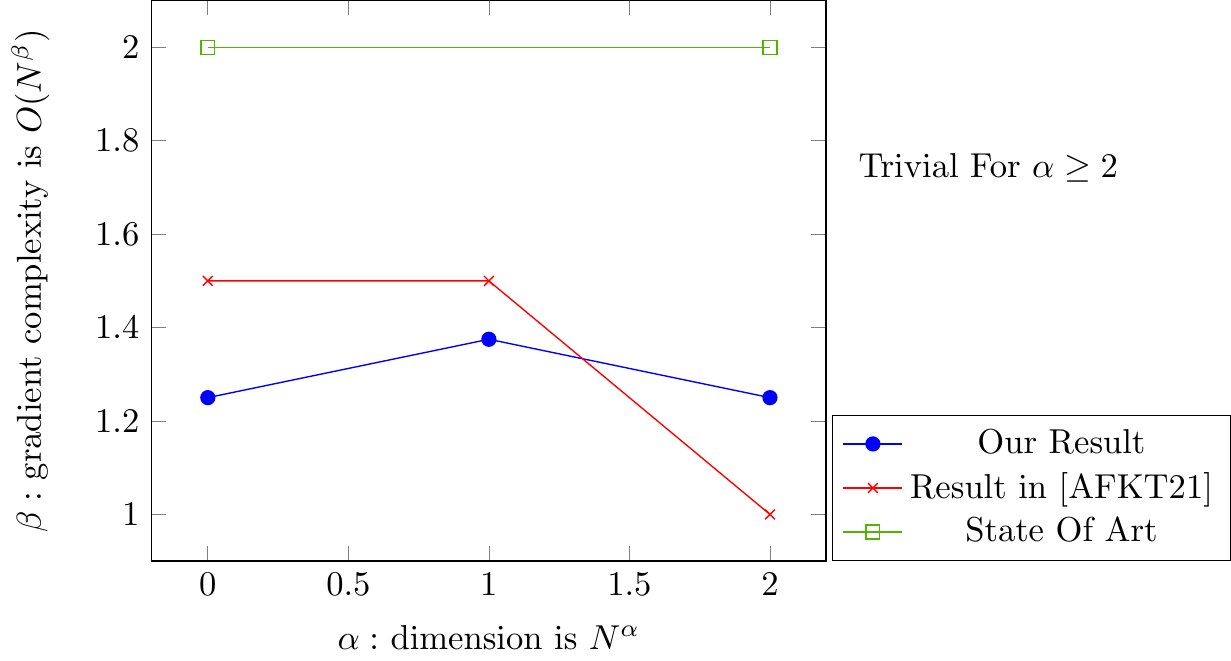}
    \caption{Comparison among our results, the recent result in \cite{afkt21} and the previous best one for the non-trivial regime ($d\leq N^2)$. Suppose $\epsilon,\delta$ are small constants. Our result is faster for the important case $d \leq N^{1+1/3}$.}
    \label{fig:compare_result}
\end{figure}

\subsection*{Road map}
We will give some basic definitions and theorems about convex optimization and differential privacy in Section~\ref{sec:prel}.
In Section~\ref{sec:framework}, we give a general algorithm framework for private convex optimization.
The results of DP-ERM are given in Section~\ref{sec:DP-ERM} and the results of DP-SCO are shown in Section~\ref{sec:DP-SCO}.
Some technical proofs are left in Appendix~\ref{sec:appendix_a}.
\section{Preliminaries}
\label{sec:prel}
In this section, we briefly recall some of the main definitions we use from the convex optimization theory and differential privacy.
We refer the readers to excellent books \cite{nes05, dr14} for more details on these topics.
\subsection{Convex Optimization}

\begin{definition}[Empirical risk minimization, Stochastic Convex Optimization]
Let ${\cal K}\subset \R^d$ be a closed convex set of diameter $D$. 
Given a family of convex loss functions $\{f(\omega,x)\}_{x\in \Xi}$ of $\omega$ over ${\cal K}$ and a set of samples $S=\{x_1,\cdots,x_n\}$ over the universe $\Xi$, the objective of Empirical Risk Minimization (ERM) is to minimize $$\HF(\omega)=\frac{1}{N}\sum_{x_i\in S}f(\omega,x_i).$$
The excess empirical loss with respect to a solution $\omega$ is defined by $\HF(\omega)-\HF^*$,
where $\HF^*=\min_{\omega\in \K}\HF(\omega)$.

Stochastic Convex Optimization (SCO) wants to output a solution $\omega$ to minimize the expected loss (also referred to {\sl population loss}) $F(\omega)-F^*$ where $F(\omega)=\E[x\sim {\cal P}]f(\omega,x)$ and $F^*=\min_{\omega\in \K}F(\omega)$.
\end{definition}

\begin{definition}[$L$-Lipschitz Continuity]
A function $f:{\cal K}\rightarrow \R$ is $L$-Lipschitz continuous over the domain ${\cal K}\subset \R^{d}$ if the following holds for all $\omega,\omega'\in {\cal K}:|f(\omega)-f(\omega')|\leq L\|\omega-\omega'\|_2$. 
\end{definition}

\begin{definition}[$\beta$-Smoothness]
A function $f:{\cal K}\rightarrow \R$ is $\beta$-smooth over the domain ${\cal K}\subset \R^{d}$ if for all $\omega,\omega'\in{\cal K}$, $\|\nabla f(\omega)-\nabla f(\omega')\|_2\leq \beta \|\omega-\omega'\|_2$.
\end{definition}

\begin{definition}[$\mu$-Strongly convex]
A differentiable function $f:\K\rightarrow \R$ is called strongly convex with parameter $\mu>0$ if the following inequality holds for all points $\omega,\omega'\in {\cal K}$,
\[
\langle \nabla f(\omega)-\nabla f(\omega'), \omega-\omega' \rangle\geq \mu \|\omega-\omega'\|_2^2.
\]
Equivalently,
\[
f(\omega')\geq f(\omega) +\nabla f(\omega)^{\top}(\omega'-\omega)+\frac{\mu}{2}\|\omega'-\omega\|_2^2.
\]
\end{definition}

\subsection{Differential Privacy}
\begin{definition}[Differential privacy]
A randomized mechanism $\M$ is $(\epsilon,\delta)$-differentially private if for any event ${\cal O}\in \mathrm{Range}(\M)$ and for any neighboring databases that differ in a single data element, one has
\begin{align*}
    \Pr[\M(S)\in {\cal O}]\leq \exp(\epsilon)\Pr[\M(S')\in {\cal O}] +\delta.
\end{align*}
\end{definition}

\begin{lemma}[Proposition 2.1 in \cite{dr14}](Post-Processing)
\label{lm:post_processing}
Let $\M: \mathbb{N}^{|\Xi|}\rightarrow R$ be a randomized algorithm that is $(\epsilon,\delta)$-differentially private. Let $f:R\rightarrow R'$ be an arbitrary randomized mapping. Then $f \circ \M:\mathbb{N}^{|\Xi|}\rightarrow R'$ is $(\epsilon,\delta)$-differentally private. 
\end{lemma}

\begin{theorem}[Basic Composition]
\label{thm:basic_com}
Let $\M_i:\mathbb{N}^{|\Xi|}\rightarrow R_i$ be $(\epsilon_i,\delta_i)$-differentially private.
Then if mechanism $\mathcal{M}_{[k]}: \mathbb{N}^{|\mathcal{X}|} \rightarrow \prod_{i=1}^{k} \mathcal{R}_{i}$ is defined to be $\mathcal{M}_{[k]}(x)=\left(\mathcal{M}_{1}(x), \ldots, \mathcal{M}_{k}(x)\right)$, then $\M_{[k]}$ is $(\sum_{i=1}^k\epsilon_i,\sum_{i=1}^k\delta_i)$-differentially private.
\end{theorem}

\section{A Meta Algorithm for DP Convex Optimization}
\label{sec:framework}
Many convex optimization algorithms with noisy first-order information have the following simple format.

\begin{algorithm2e}[H]
\caption{Meta Algorithm $\mathsf{META}$}
\label{alg:meta}
{\bf Input:} The objective convex function $F(\omega)$ we want to minimize, an initial point $\omega_0$.

{\bf Process:}
\For{phases $t=1,\cdots,$}
{
Get the noisy gradient $G_t\approx \nabla F(\omega_{t-1})$\;
Update the result by some sub-procedure: $\omega_t\leftarrow $ Sub-procedure$(\omega_{t-1},G_t)$\;
}

{\bf Output:}
Some function of $\{\omega_i\}_{i\geq 1}$. 
\end{algorithm2e}

We can use the above algorithmic framework to solve ERM privately.
Specifically, we make two simple modifications to make it private.
First, we compute gradients over a uniform sample of some size $B$.
Next, we add a carefully calibrated Gaussian noise to these gradients and take average, before updating our results.
This gives us a meta differentially private algorithm for convex optimization problems, and is described in Algorithm~\ref{alg:dp_meta}.  
The DP analysis then follows from a careful accounting of the privacy budget lost in each iteration, and the bound on excess empirical risk comes from the property of the optimization algorithm.

\begin{algorithm2e}[H]
\caption{Private Meta Algorithm $\mathsf{META_{DP}}$}
\label{alg:dp_meta}
{\bf Input:} Sample set $S=\{x_1,\cdots,x_N\}$, the objective convex function $F(\omega)$ we want to minimize, the initial point $\omega_0$, and privacy parameter $\epsilon,\delta$\;

{\bf Process:}
\For{phases $t=1,\cdots,T$}
{
Select a random sample set $S_t$ from the uniform distribution over all subsets of $S$ of size $B$\;

Let $G_t=(\sum_{x_i\in S_t}\nabla f(\omega_{t-1},x_i)+v)/B$, where $v\sim\N(0,\sigma^2I_{d\times d})$\label{ln:add_gaussian}\;

Update the result by some sub-procedure $\omega_t\leftarrow $ Sub-procedure$(\omega_{t-1},G_t)$\;
}

{\bf Output:} Some function of $\{\omega_i\}_{i\geq 1}$. 

\end{algorithm2e}

The above framework is a sub-sampled Gaussian mechanism, for which we can use tCDP proposed in \cite{bdrs18} to analyze its privacy guarantee.
As this is a direct application of the main result in \cite{bdrs18}, we leave the proof of the following theorem in the Appendix.
\begin{theorem}
\label{thm:dp_meta}
Suppose $\{f(\cdot,x)\}_{x\in \Xi }$ is a family of $G$-Lipschitz and convex functions over $\K$, for $\epsilon< c_1 B^2T/N^2, B\leq N/10$ and $1/2\geq\delta>0$, by setting $\sigma=\frac{c_2GB\sqrt{T\log(1/\delta)}}{\epsilon N}$ for some constant $c_1$ and $c_2$, $\mathsf{META_{DP}}$ is $(\epsilon,\delta)$-differential private.
\end{theorem}
 
\section{Differentially Private ERM}
\label{sec:DP-ERM}
In this section, we present private algorithms achieving the optimal excess empirical loss with subquadratic gradient complexity when the dimension is super constant.
We consider non-smooth strongly-convex functions first, and then show how to reduce the general non-smooth case to the strongly-convex case in the last subsection.

\subsection{Non-smooth Strongly-convex Functions}
We use the framework introduced in Section~\ref{sec:framework} and give a faster private algorithm. 
Specifically, we modify a stochastic convex optimization algorithm in \cite{gl12} to fit into our framework. 
First we recall some properties of that algorithm.


Suppose $f:{\cal K}\rightarrow \R$ is a convex function, and the objective is to get 
\begin{align*}
    \Psi^*:=\min_{\omega\in \K}\{\Psi(\omega)=f(\omega)+h(\omega)\},
\end{align*}
where $\K$ is a closed convex set and $h(\omega)$ is a simple convex function with known structure.

\begin{theorem}[Proposition 9 in \cite{gl12}]
\label{thm:fast_alg_gl12}
If the following conditions are met:
\begin{itemize}
    \item For some $L\geq0,M\geq 0$ and $\mu> 0$,
\begin{equation*}
    \label{eq:preconditon}
    \begin{split}
       &~ \frac{\mu}{2}\|y-\omega\|^2_2\leq  f(y)-f(\omega)-\langle  g(\omega),y-\omega\rangle
    \leq  \frac{L}{2}\|y-\omega\|^2_2+M\|y-\omega\|_2, \quad \forall \omega,y\in{\cal K},
    \end{split}
\end{equation*}
where $g(\omega)\in\partial f(\omega)$ and $\partial f(\omega)$ denotes the sub-differential of $f$ at $\omega$.

\item For each call of the stochastic oracle $\mathcal{G}$ with the input $\omega_t\in {\cal K}$, the stochastic oracle $\mathcal{G}$ can output an independent vector $\mathcal{G}(\omega_t)$ 
such that $\E[\mathcal{G}(\omega_t)]\in \partial f(\omega_t)$. 

\item For any $t\geq 1$ and $\omega_t\in \K$, $\E[\|\mathcal{G}(\omega_t)-g(\omega_t)\|_2^2]\leq V$.
\end{itemize}
Then after $T$ iterations, Algorithm~\ref{alg:gl12} given below outputs $\omega_T$ such that
\begin{align*}
    \E[\Psi(\omega_T)-\Psi^*]\leq O\left(\frac{L\|\omega_0-\omega^*\|_2^2}{T^2}+\frac{M^2+V}{\mu T}\right),
\end{align*}
where $\omega^*=\arg\min_{\omega\in \K}\Psi(\omega)$ and $\Psi^*=\Psi(\omega^*)$.
\end{theorem}


\begin{algorithm2e}[H]
\caption{Accelerated stochastic approximation (AC-SA) algorithm}
\label{alg:gl12}

{\bf Input:} Initial point $\omega_0\in \K$.

{\bf Initialization:}
Set the initial point $\omega_0^{ag}=\omega_0$\;
Set the step-size parameters $\alpha_t=\frac{2}{t+2}$ and $\gamma_t=\frac{4L}{t(t+1)}$\;

{\bf Process:}

\For{$t=1,\cdots,T$}
{
Let $\omega_t^{md}=\frac{(1-\alpha_t)(\mu+\gamma_t)}{\gamma_t+(1-\alpha_t^2)\mu}\omega_{t-1}^{ag}+\frac{\alpha_t[(1-\alpha_t)\mu+\gamma_t]}{\gamma_t+(1-\alpha_t^2)\mu}\omega_{t-1}$\;

Query Oracle $\mathcal{G}_t\equiv \mathcal{G}(\omega_t^{md})$\;

Let $\omega_t=\arg\min_{\omega\in \K}\{\alpha_t[\langle \mathcal{G}_t,\omega\rangle+h(\omega)+\mu \|\omega_t^{md} - \omega \|_2^2]+[(1-\alpha_t)\mu+\gamma_t]\|\omega_{t-1}-\omega\|_2^2 \}$\;

 $\omega_t^{ag}=\alpha_t\omega_t+(1-\alpha_t)\omega_{t-1}^{ag}$\;
}

{\bf Return:} $\omega_T^{ag}$.

\end{algorithm2e}

\subsubsection{Smoothing function}
From the statement of Theorem~\ref{thm:fast_alg_gl12}, it is clear that the Algorithm~\ref{alg:gl12} gives much better convergence rates for smooth functions.
As we are considering non-smooth functions, we need an efficient way to smooth the objective function without introducing too much error.
In the next few paragraphs, we show how to achieve that.

Recall that $D$  denotes the diameter of the closed convex set $\K\subset \R^d$. 
Suppose $\{f(\cdot,x)\}_{x\in \Xi }$ is a family of $G$-Lipschitz and $\mu$-strongly convex functions over $\K$. 
This implies that for any sample set $S$, the empirical loss function $\HF(\omega)
$ we consider is $G$-Lipschitz and $\mu$-strongly convex over the domain ${\cal K}$. 

We do a convolution on $f(\cdot,x)$, which is denoted by $f(\cdot,x)*n_{r}$. 
The objective function after the convolution step becomes
$\Fnr(\omega)=\frac{1}{N}\sum_{x_i\in S}\E_{y\sim n_r}f(\omega+y,x_i)$, where $n_r$ is the uniform density on the $\ell_2$ ball of radius $r$. 
By Lemma 7 and Lemma 8 in \cite{yns11} while the forth result on forth item was supplemented by Lemma E.2 in \cite{dbw12}, we know the claim below.

\begin{claim}
\label{clm:convolution_smooth}
Suppose $\{f(\cdot,x)\}_{x\in \Xi }$ is convex and $G$-Lipschitz over $\K+B_2(0,r)$. For $\omega\in \K$, $\Fnr(\omega)$ has following properties:
\begin{itemize}
    \item $\HF(\omega)\leq \Fnr(\omega)\leq \HF(\omega)+ Gr$;
    \item $\Fnr(\omega)$ is $G$-Lipschitz;
    \item $\Fnr(\omega)$ is $\frac{G\sqrt{d}}{r}$-Smooth;
    \item For random variables $y\sim n_r$ and $x$ uniformly from $S$, one has
    \[
    \E[\nabla f(\omega+y,x)]=\nabla \Fnr (\omega) \quad\]
    and
    \[\quad \E[\|\nabla \Fnr(\omega)-\nabla f(\omega+y,x) \|_2^2 ]\leq G^2.
    \]
\end{itemize}
\end{claim}

Furthermore, the convolution operation preserves strong convexity, which implies the fact below.

\begin{fact}
\label{fact:conv_strongly}
Let $n_r$ be the uniform density on the $\ell_2$ ball of radius $r$, and $f:  \K_r\rightarrow \R$ be a $\mu$-strongly convex function over $\K_r$. 
Then $\E_{y\sim n_r}f(y+\cdot)$ is $\mu$-strongly convex over $\K$.
\end{fact}

\subsubsection{Algorithm}
Now we state the our modifications to make $\ACSA$ private and prove its properties.
Recall that $y\sim n_r$ is a $d$-dimension vector drawn from the uniform density on the $\ell_2$ ball of radius $r$.
We start with the description of our algorithm.

\begin{algorithm2e}[H]
\caption{Private $\ACSA$}
\label{alg:main}
{\bf Input:} A convex set $\K$ with diameter $D$, a family $\{f(\cdot,x_i)\}_{i\in [N] }$ of $G$-Lipschitz and $\mu$-strongly convex functions over $\K$, an initial point $\omega_0\in \K$, privacy parameters $\epsilon,\delta$, the batch size $B$, and the number of steps $T$.

Set $r\leftarrow\frac{D}{Td^{1/4}}$ and $\sigma\leftarrow \Theta(\frac{GB\sqrt{T\log(1/\delta)}}{\epsilon N})$\;

Run the $\ACSA$ with the Oracle $\mathcal{G}$ defined below\;

{\bf Return:} The output of $\ACSA$\\

\BlankLine \BlankLine
{\bf Oracle $\mathcal{G}(\omega)$:}

 Select a random sample set $S_t$ from the uniform distribution over all subsets of $S$ of size $B$.

{\bfseries Return:} $\big(\sum_{x_i\in S_t}\partial f(\omega+y,x_i)+v)/B$, where $y\sim n_r$ and $v\sim \N(0,\sigma^2\I_{d\times d}\big)$.
\end{algorithm2e}

\subsubsection{Utility and Privacy}

It is not hard to show that Private \ACSA (Algorithm~\ref{alg:main}) is an instance of \METADP (see Section~\ref{sec:framework}), so we have the following guarantee directly by Theorem~\ref{thm:dp_meta}.

\begin{lemma} \label{lem:ACSA_DP}
For $\epsilon\leq c_1B^2T/N^2,\delta\leq 1/2, B\leq N/10$ and $\sigma=\frac{c_2GB\sqrt{T\log(1/\delta)}}{\epsilon N}$ where $c_1\leq 1,c_2\geq 1$ are constants, Private \ACSA is $(\epsilon,\delta)$-DP.
\end{lemma}

Now we consider the accuracy of Private \ACSA. 

\begin{lemma}
\label{lm:guarantee_ACSA}
Under the assumptions defined in Algorithm Private \ACSA, 
after $T$ iterations, it outputs $\omega_T$ such that 
\begin{align*}
    \E[\HF(\omega_T)-\HF^{*}]=O\left(\frac{G^2/B+\sigma^2d/B^2}{\mu T}+\frac{GDd^{1/4}}{T}\right),
\end{align*}
where $\omega^*=\arg\min_{\omega\in \K}\HF(\omega)$, and $\HF^{*}=\min_{\omega}\HF(\omega)$.
\end{lemma}
\begin{proof}
By Claim~\ref{clm:convolution_smooth}, we know that $\Fnr$ is $G$-Lipschitz and $\frac{G\sqrt{d}}{r}$-smooth. 
Furthermore, by Fact~\ref{fact:conv_strongly}, we know that $\Fnr$ is $\mu$-strongly convex. 
For any $t$th iteration, one has that $\E[\mathcal{G}_t]=\nabla \Fnr(\omega_t^{md})$ and $\E[\|\mathcal{G}_t-\nabla \Fnr(\omega_t^{md})\|_*^2]\leq G^2/B+\sigma^2d/B^2$.  
Then by Theorem~\ref{thm:fast_alg_gl12} with $M = 0, L = \frac{G\sqrt{d}}{r}, V = G^2/B+\sigma^2d/B^2$, we get
\begin{align*}
    \E[\Fnr(\omega_T)-\min_{\omega}\Fnr(\omega)]=O\left(\frac{G^2/B+\sigma^2d/B^2}{\mu T}+\frac{GD^2\sqrt{d}}{T^2r}\right).
\end{align*}

Next, by the first bullet of Claim~\ref{clm:convolution_smooth}, we know that $\HF(\omega)\leq \Fnr(\omega)\leq \HF(\omega)+Gr$ for any $\omega$. 
Combining these together, we get
\begin{align*}
    &\E[\HF(\omega_T)-\HF(\omega^*)]\\
    =& \E[\HF(\omega_T)-\Fnr(\omega_T)]+
    \E[\Fnr(\omega_T)-\min_{\omega}\Fnr(\omega)] 
     +  \min_{\omega}\Fnr(\omega)-\HF(\omega^*)\\
    \leq & 2 Gr+O(\frac{G^2/B+\sigma^2d/B^2}{\mu T}+\frac{GD^2\sqrt{d}}{T^2r}).
\end{align*}
By setting $r=\frac{Dd^{1/4}}{T}$, we completes the proof.
\end{proof}

Before stating the main result of this section, we prove two technical lemmas that can remove the dependence on the diameter term. 
Lemma~\ref{lm:recurrence_converge} below is used to prove Lemma~\ref{lm:remove_dia}.

\begin{lemma}
\label{lm:recurrence_converge}
Consider a sequence $x_1,x_2,\cdots$. Suppose $0\leq x_1\leq n$ and $0\leq x_{i+1}\leq \sqrt{x_i}+1$, then for $k \geq \lceil \log\log n\rceil$, one has that $x_k\leq 16$.
\end{lemma}

\begin{proof}
Without loss of generality, let $x_{i+1}=\sqrt{x_i}+1$.

We construct another sequence $y_1,\cdots,y_k$ such that $y_1=x_1$ and $y_{i+1}= 2\sqrt{y_i}$. Then by induction, it is easy to prove that for each $i\in [k]$, $y_i\geq x_i$. So we only need to prove that $y_k\leq 16$.

Let $z_i=\log_2 y_i$, then one has $z_{i+1}=z_i/2+1$. Obviously, we know that $z_{i}=2^{-i+1}(z_1-2)+2$ and $z_{k}\leq 4$, which means that $x_k\leq y_k\leq  16$.
\end{proof}

Recall that the lower bound of strongly convex case is $\Omega(\frac{G^2}{\mu}+\frac{G^2d\log(1/\delta)}{\mu \epsilon^2N^2})$ while for the general case is $\Omega(GD+\frac{GD\sqrt{d\log(1/\delta)}}{\epsilon N})$. 
Therefore, we only need to think about the case when $\frac{d\log(1/\delta)}{\epsilon^2N^2}\leq 1$, or the bound will be trivial. 
The following lemma says if we can achieve sum of these two lower bounds for strongly-convex case, then we can achieve the optimal bound for the strongly-convex case, which implies we can reduce the Strongly-Convex Case to General Convex Case.

\begin{lemma}[Reduction to General Convex Case]
\label{lm:remove_dia} Given $\hat{F}$ is $G$-Lipschitz and $\mu$-strongly convex. 
Suppose for any $\epsilon,\delta<1/2$, we have an $(\epsilon,\delta)$-differentially private algorithm $\A$ which takes $\omega_0$ as the initial start point and
outputs a solution $\omega_T$ such that
\begin{align*}
    \E[\HF(\omega_T)-\HF^*]=O\left(\frac{G^2d\log(1/\delta)}{\mu\epsilon^2 N^2}+ \frac{GD\sqrt{d\log(1/\delta)}}{\epsilon N}\right),
\end{align*}
where $\omega^*=\arg\min_{\omega\in\K}\HF(\omega)$ and $D=\|\omega_0-\omega^*\|_2$.
Then by taking $\A$ as sub-procedure with some modifications on parameters, we can get an $(\epsilon,\delta)$-differentially private solution  with excess empirical loss at most
\begin{align*}
\E[\HF(\omega_T)-\HF^*]=O\left(\frac{G^2d\log(1/\delta)}{\mu\epsilon^2 N^2}\right).
\end{align*}
Furthermore, if $\A$ uses $g(N,\epsilon,\delta)$ many gradients, the new algorithm uses $\sum_{i\ge 1} g(N,\epsilon/2^i,\delta/2^i)$ many gradients.
\end{lemma}
\begin{remark} All algorithms in this paper uses less gradients if $\epsilon$ and $\delta$ are smaller. So, the new algorithm uses essentially as much as the given algorithm.
\end{remark}
\begin{proof}
Repeat the private algorithm $\A$ for $k= \lceil \log\log N^3\rceil$ times. For the $i$th repetition, we start from the output of the last repetition 
and use $\A$ as a sub-procedure with privacy parameter $\epsilon_i=\epsilon/2^{k+1-i}$ and $\delta_i=\delta/2^{k+1-i}$. (Note that the noise is decreasing so that the last step gives the best solution).
We show that the last output has excess empirical risk at most $O\left(\frac{G^2d\log(1/\delta)}{\mu\epsilon^2 N^2}\right)$.

More specifically, let $\omega_i$ be the output of the $i$th repetition, $\Delta_i=\E[\HF(\omega_i)-\HF^*]$ and $D_i^2=\E[\|\omega_i-\omega^*\|^2]$.
As the objective function $\HF$ is $\mu$-strongly convex, we know that $\frac{1}{2}\mu D_i^2\leq  \Delta_i$ for all $i\geq 0$.

By the guarantee of the algorithm, there exists some constant $c\geq 1$ such that
\begin{align*}
    \Delta_{i+1} =& \E[\HF(\omega_{i+1})-\HF^*]\\
    \leq  & c\frac{GD_i\sqrt{d\log(1/\delta_i)}}{\epsilon_i N}+c\frac{G^2d\log(1/\delta_i)}{\mu\epsilon_i^2 N^2}\\
    \leq & c\frac{G\sqrt{d\log(1/\delta_i)}}{\epsilon_i N}\sqrt{\frac{2\Delta_i}{\mu}}+\frac{E_i}{c},
\end{align*}
where we define $E_i=2c^2\frac{G^2d\log(1/\delta_i)}{\mu\epsilon_i^2 N^2}$.

As $E_i/E_{i+1}=\frac{\epsilon^2_{i+1}\log(1/\delta_i)}{\epsilon^2_i\log(1/\delta_{i+1})}\leq 8$, we can rearrange the above function and get
\begin{align*}
    \frac{\Delta_{i+1}}{64E_{i+1}}\leq & \frac{\sqrt{\Delta_iE_i}+\frac{E_i}{c}}{64E_{i+1}} \\
    \leq & \frac{E_i}{64E_{i+1}}(\sqrt{\frac{\Delta_i}{E_i}}+\frac{4}{c}) \\
    \leq & \sqrt{\frac{\Delta_i}{64E_i}}+1.
\end{align*}

By strong convexity one has that $\Delta_1\leq G^2/\mu$,  and $E_1 = \Omega(G^2 \log^3 N/(\mu N^2)) = \Omega(G^2/(\mu N^3))$ by the definition, so $\Delta_1/E_1\leq N^3$. 
Then by Lemma~\ref{lm:recurrence_converge}, after $k=\lceil \log\log N^3 \rceil$ repetitions, we get $\frac{\Delta_k}{64E_k}\leq 16$.
This further implies that there is a  solution with expected error
\begin{align*}
    \E[\HF(\omega_k)-\HF^*] =  
    O(\frac{G^2d\log(1/\delta)}{\mu\epsilon^2 N^2}).
\end{align*}

The privacy guarantee comes directly from the basic composition theorem (See Theorem~\ref{thm:basic_com}).
\end{proof}

We did not optimize constants in the calculations above.
Now we are ready to state the main result for the strongly-convex case.

\begin{theorem}[Strongly Convex Case for Theorem~\ref{thm:DP-ERM}]
\label{thm:strongly_SCO}
Suppose ${\cal K}\subset \R^d$ is a closed convex set of diameter $D$ and $\{f(\cdot,x)\}_{x\in \Xi }$ is a family of $G$-Lipschitz and $\mu$-strongly convex functions over $\K_r$ where $r=\frac{D\sqrt{d\log(1/\delta)}}{\epsilon N}$. 
For $\epsilon,\delta\leq 1/2$, given any sample set $S$ consists of $N$ samples from $\Xi$ and arbitrary initial point $\omega_0\in \K$, we have an $(\epsilon,\delta)$-differentially private algorithm $\A$ which takes $O(\frac{\epsilon N^{\frac{3}{2}}}{d^{1/8}\log^{1/4}(1/\delta)}+\frac{\epsilon^2 N^2}{d\log(1/\delta)})$ gradient queries and outputs $\omega_T$ such that
\begin{align*}
    \E[\HF(\omega_T)-\HF^*]=O\left(\frac{G^2d\log(1/\delta)}{\mu\epsilon^2N^2}\right),
\end{align*}
where $\HF(\omega)=\frac{1}{N}\sum_{x_i\in S}f(\omega,x_i), \HF^*=\min_{\omega\in \K}\HF(\omega)$, and the expectation is taken over the randomness of the algorithm itself. 
\end{theorem}

\begin{proof}
By Lemma~\ref{lm:guarantee_ACSA}, the output $\omega$ of Private $\ACSA$ satisfies
\begin{align*}
    \E[\HF(\omega)-\HF^{*}]=O\left(\frac{\frac{G^2}{B}+\frac{\sigma^2d}{B^2}}{\mu T}+\frac{GDd^{1/4}}{T}\right).
\end{align*}

By setting $\sigma=\frac{c_2GB\sqrt{T\log(1/\delta)}}{\epsilon N}$ and $T=\lceil \frac{100 \epsilon N}{c_1 d^{1/4}\sqrt{\log(1/\delta)}} \rceil$ ($c_1, c_2$ are defined in Lemma~\ref{lem:ACSA_DP}), one has
\begin{align*}
    \E[\HF(\omega)-\HF^{*}] & = O\left(\frac{G^2}{\mu BT}+\frac{G^2d\log(1/\delta)}{\mu \epsilon^2 N^2}+\frac{GDd^{1/4}}{T}\right) \\
    & = O\left(\frac{G^2}{\mu BT}+\frac{G^2d\log(1/\delta)}{\mu \epsilon^2 N^2}+\frac{GD\sqrt{d \log(1/\delta)}}{\epsilon N}\right)
\end{align*}

To ensure that Private $\ACSA$ is $(\epsilon,\delta)$-DP, we set $B = \lceil\sqrt{\frac{\epsilon N^2}{c_1 T}}+\frac{\epsilon^2N^2}{d\log(1/\delta)T}\rceil$. By our choice of $T$, we have $B \leq N/10$ and $\epsilon \leq c_1 B^2 T / N^2$. Hence, we can apply Lemma~\ref{lem:ACSA_DP} to conclude the guarantee of  $(\epsilon,\delta)$ differential privacy.

%
Furthermore, we get a solution $\omega$ such that
\begin{align*}
    \E[\HF(\omega)-\HF^*]=O\left(\frac{G^2d\log(1/\delta)}{\mu\epsilon^2 N^2}+\frac{GD\sqrt{d\log(1/\delta)}}{\epsilon N} \right).
\end{align*}

As for the total gradient complexity of our algorithm, we are under the assumption that $\frac{d\log(1/\delta)}{\epsilon^2N^2}\leq 1$, which means that $\frac{\epsilon N}{d^{1/4}\sqrt{\log(1/\delta)}}\geq d^{1/4}$, and $T=\lceil \frac{100 \epsilon N}{c_1 d^{1/4}\sqrt{\log(1/\delta)}} \rceil=\Theta(\frac{\epsilon N}{d^{1/4}\sqrt{\log(1/\delta)}})$.
As for the batch size, we know $\sqrt{\frac{\epsilon N^2}{T}}+\frac{\epsilon^2N^2}{d\log(1/\delta)T}=\omega(1)$ and thus $B=\lceil\sqrt{\frac{\epsilon N^2}{T}}+\frac{\epsilon^2N^2}{d\log(1/\delta)T}\rceil=\Theta(\sqrt{\frac{\epsilon N^2}{T}}+\frac{\epsilon^2N^2}{d\log(1/\delta)T})$, from which we get the gradient complexity is 
\begin{align*}
    BT=\Theta\left(\frac{\epsilon N^{\frac{3}{2}}}{d^{1/8}\log^{1/4}(1/\delta)}+\frac{\epsilon^2 N^2}{d\log(1/\delta)}\right).
\end{align*}

By Lemma~\ref{lm:remove_dia} we can adjust Private \ACSA and get a final solution $\omega_{T}$ such that
\begin{align*}
    \E[\HF(\omega_T)-\HF^*]=O\left(\frac{G^2d\log(1/\delta)}{\mu\epsilon^2N^2}\right),
\end{align*}
with gradient complexity $\Theta(\sum_{i=1}^{\log\log N^3}\frac{(\epsilon/2^i)N^{3/2}}{d^{1/8}\log^{1/4}(2^i/\delta)}+\frac{(\epsilon/2^i)^2N^2}{d\log(2^i/\delta)} )=\Theta(\frac{\epsilon N^{\frac{3}{2}}}{d^{1/8}\log^{1/4}(1/\delta)}+\frac{\epsilon^2 N^2}{d\log(1/\delta)}),$
which completes the proof.
\end{proof}

\subsection{General Non-smooth Convex Functions}
In the general non-smooth case, we only assume that the family of functions $\{f(\cdot,x)\}_{x\in \Xi}$ is $G$-Lipschitz and convex over $\K$. 
We now give a reduction from this case to the strongly-convex case, which completes our second main result.

\begin{lemma}
\label{lm:convex_to_strongly}
Suppose ${\cal K}\subset \R^d$ is a convex set of diameter $D$ and let $\{f(\cdot,x)\}_{x\in \Xi }$ be a family of convex functions over $\K$, which are $G$-Lipschitz and $\mu$-strongly convex. 
Given any sample set $S$ consists of $N$ samples from $\Xi$ and other necessary inputs, suppose we have a $(\epsilon,\delta)$-DP algorithm $\A$ which can output a solution $\omega_T$ such that
\begin{align*}
    \E[\HF(\omega_T)-\HF^*]=O\left(\frac{G^2d\log(1/\delta)}{\mu\epsilon^2N^2}\right),
\end{align*}
where $ \HF^*=\min_{\omega\in \K}\HF(\omega)$.

Then when $\{h(\cdot,x)\}_{x\in \Xi }$  is only $G$-Lipschitz and convex with necessary inputs, for any sample set $S$ of size $N$, we also have a $(\epsilon,\delta)$-DP algorithm $\A'$ which can get a solution $\omega_T$ such that 
\begin{align*}
     \E[\hat{H}(\omega_T)-\hat{H}^{*}]=O\left(\frac{GD\sqrt{d\log(1/\delta)}}{\epsilon N}\right).
\end{align*}
where $\hat{H}(\omega)=\frac{1}{N}\sum_{x_i\in S}h(\omega,x_i), \hat{H}^*=\min_{\omega\in \K}H(\omega)$. The gradient complexity and privacy guarantee of $\A$ and $\A'$ are the same.

Moreover, the reduction also holds in the context of SCO.
\end{lemma}

\begin{proof}
We only consider this lemma in the context of ERM, as we can use the nearly the same argument for SCO.

The proof of this reduction is rather simple: After getting $\{h(\cdot,x_i)\}_{x_i\in S }$, we only need to consider $h_u(\omega,x)=h(\omega,x)+u\|\omega\|^2$. 
Then $h_{u}(\cdot,x)$ is $u$-strongly convex and $O(G+uD)$-Lipschitz for any $x$ with $\|x\|_2 \leq 2 D$. 

For the case $uD\leq G$, we run $\A$ on $\{h_u(\cdot,x_i)\}_{x_i\in S }$ to get a solution $\omega_T$ with loss
\begin{align*}
    \E[H_u(\omega_T)-H^*_u]=O\left(\frac{G^2d\log(1/\delta)}{u\epsilon^2N^2}\right),
\end{align*}
where $H_u(\omega)=\frac{1}{N}\sum_{x_i\in S}h(\omega,x_i)+u\|\omega\|^2$ and $H_u^*=\min_{\omega\in \K}H_u(\omega)$.
Now by setting $u=\Theta\left(\frac{G\sqrt{d\log(1/\delta)}}{D\epsilon N}\right)$, one has
\begin{align*}
    \E[\hat{H}(\omega_T)-\hat{H}^*]
    = & O\left(\frac{G^2d\log(1/\delta)}{u\epsilon^2N^2}+uD^2\right)\\
    = & O\left(\frac{GD\sqrt{d\log(1/\delta)}}{\epsilon N}\right).
\end{align*}

For the case $uD\geq G$, we have $\frac{GD\sqrt{d\log(1/\delta)}}{\epsilon N} \geq GD$ and hence we can simply output the initial point $\omega_0$ as the solution with a loss no more than $GD$.
\end{proof}

The above reduction completes the main result of this subsection.

\begin{theorem}
[General Convex Case for Theorem~\ref{thm:DP-ERM}]
Suppose ${\cal K}\subset \R^d$ is a convex set of diameter $D$ and $\{f(\cdot,x)\}_{x\in \Xi }$ is a family of $G$-Lipschitz and convex functions over $\K_r$ where $r=\frac{D\sqrt{d\log(1/\delta)}}{\epsilon N}$. For $\epsilon,\delta\leq 1/2$, given any sample set $S$ consists of $N$ samples from $\Xi$ and arbitrary initial point $\omega_0\in \K$, we have a $(\epsilon,\delta)$-differentially private algorithm $\A$ which takes $O\left(\frac{\epsilon N^{\frac{3}{2}}}{d^{1/8}\log^{1/4}(1/\delta)}+\frac{\epsilon^2 N^2}{d\log(1/\delta)}\right)$ gradient queries and outputs $\omega_T$ such that
\begin{align*}
    \E[\HF(\omega_T)-\HF^*]=O\left(\frac{GD\sqrt{d\log(1/\delta)}}{\epsilon N}\right),
\end{align*}
where $\HF^*=\min_{\omega\in \K}\HF(\omega)$, and the expectation is  taken over
the randomness of the algorithm. 
\end{theorem}

\section{Differentially Private SCO}
\label{sec:DP-SCO}
In this section we study SCO.
We can use the iterative localization technique in \cite{fkt20} to reduce the SCO problem to an ERM problem.
More specifically, if we can solve private ERM and get a (nearly) optimal empirical loss, then we can solve private SCO with (nearly) optimal excess population loss with the following algorithm framework (Algorithm~\ref{alg:local}).
See Theorem~\ref{thm:local} for the corresponding formal statement.

\begin{algorithm2e}
\caption{Iterative Localized Algorithm Framework $\A'$}
\label{alg:local}
{\bf Input:} A family of $G$-Lipschitz and $\mu$-strongly convex function $f:\K\times\Xi\rightarrow \R$, initial point $\omega_0\in \K$ and privacy parameter $\epsilon,\delta$.

{\bf Process:}
Set $k=\lceil \log N \rceil$\;
\For{$i=1,\cdots,k$}
{
Set $\epsilon_i=\epsilon/2^i,N_i=N/2^{i},\eta_i=\eta/2^{5i}$\;

Apply $(\epsilon_i,\delta_i)$-DP ERM algorithm $\A_{\epsilon_i,\delta_i}$ over $\K_{i}=\left\{\omega \in \K:\left\|\omega-\omega_{i-1}\right\|_{2} \leq 2G\eta_iN_i\right\}$ \ with the function
$\HF_i(\omega)=\frac{1}{N_i}\sum_{j\in S_i}f(\omega,x_j)+\frac{1}{\eta_iN_i}\|\omega-\omega_{i-1}\|^2$ where $S_i$ consists of $N_i$ samples with replacement from $\mathcal{P}$\;

Let $\omega_i$ be the output of the ERM algorithm\;

}

{\bf Return:} The final iterate $\omega_k$\;
\end{algorithm2e}

\begin{theorem}
\label{thm:local}
Suppose we have an algorithm $\A$ which can solve ERM under strongly convex case and gets a solution with excess empirical loss $O(\frac{G^2}{\mu N})$ by using $g(N)$ many gradients, then we  have an algorithm $\A'$ which can solve SCO under general case and gets a solution with excess population loss $O(\frac{GD}{\sqrt{N}})$ by using $\sum_{i=1}^{\lceil \log N\rceil}g( N/2^i)$ many gradients.

Moreover, for $\epsilon,\delta\leq 1/2$, if $\A_{\epsilon,\delta}$ is $(\epsilon,\delta)$-differentially private with excess empirical loss $O(\frac{G^2}{\mu}(\frac{1}{N}+\frac{d\log(1/\delta)}{\epsilon^2 N^2}))$ under the strongly convex case by using $g(N,\epsilon, \delta)$ many gradients, then we can get $(\epsilon, \delta)$-differentially private $\A'$ with excess population loss
$O(GD(\frac{1}{\sqrt{N}}+\frac{\sqrt{d\log(1/\delta)}}{\epsilon N}))$ by querying gradients at most $\sum_{i=1}^{\lceil \log N\rceil}g( N/2^i,\epsilon/2^i,\delta/2^i)$ times.
\end{theorem}


We only prove the bound with privacy guarantee, as the (non-private) bound can be proved with similar argument.
Two technical lemmas will be proved at first, after which we will complete the proof.

\begin{lemma}
\label{lm:accuray_phase}
Let $\homega_i=\arg\min_{\omega\in \K}\HF_i(\omega)$, then
\begin{align*}
    \E[\|\omega_i-\homega_i\|_2^2]\leq O(\frac{G^2\eta_i^2d\log(1/\delta_i)}{\epsilon^2_i}+G^2\eta_i^2N_i).
\end{align*}

\end{lemma}

\begin{proof}
At first, we prove that $\homega_i\in\K_i$. The definition of $\homega_i$ implies that
\begin{align*}
    \frac{1}{N_{i}} \sum_{j=1}^{N_{i}} f\left(\homega_{i}, x_{j}\right)+\frac{1}{\eta_{i} N_{i}}\left\|\homega_{i}-\omega_{i-1}\right\|_{2}^{2} \leq \frac{1}{N_{i}} \sum_{j=1}^{N_{i}} f\left(\omega_{i-1}, x_{j}\right).
\end{align*}
Then we know that
\begin{align*}
    \frac{1}{\eta_iN_i}\|\homega_i-\omega_{i-1}\|_2^2\leq G\|\homega_i-\omega_{i-1}\|_2,
\end{align*}
which implies $\homega_i \in \K_i$.

Next, note that $\HF_i$ is $\lambda_i=\frac{1}{\eta_iN_i}$-strongly convex, by the guarantee of our ERM algorithm, we know that
\begin{align*}
    \frac{\lambda_i}{2}\E[\|\homega_i-\omega_i\|_2^2]\leq & \E[\HF_i(\homega_i)-\HF_i(\omega_i)]\\
    \leq & O(\frac{G^2d\log(1/\delta_i)}{\lambda_i\epsilon_i^2N_i^2}+\frac{G^2}{\lambda_iN_i})\\
    =& O(\frac{G^2\eta_id\log(1/\delta_i)}{\epsilon_i^2N_i}+G^2\eta_i),
\end{align*}
which implies 

\begin{align*}
    \E[\|\homega_i-\omega_i\|_2^2]\leq O(\frac{G^2\eta_i^2d\log(1/\delta_i)}{\epsilon^2_i}+G^2\eta_i^2N_i).
\end{align*}
\end{proof}

\begin{lemma}
For any $y\in \K$, we know that
\begin{align*}
    \E[F(\homega_i)-F(y)]\leq \frac{\E[\|\omega_{i-1}-y\|_2^2]}{\eta_iN_i}+O(G^2\eta_i).
\end{align*}
\end{lemma}
\begin{proof}
Let $r(\omega,x)=f(\omega,x)+\frac{1}{\eta_i N_i}\|\omega-\omega_{i-1}\|_2^2$, $R(\omega)=\E_{x\sim {\cal P}}r(\omega,x)$ and $y^*=\arg\min_{\omega\in\K}R(\omega)$. By Theorem 6 in \cite{sssss09}, one has that
\begin{align*}
    \E[R(\homega_i)-R(y)]=& \E[F(\homega_i)+\frac{1}{\eta_i N_i}\|\homega_i-\omega_{i-1}\|_2^2-F(y)-\frac{1}{\eta_i N_i}\|y-\omega_{i-1}\|_2^2]\\
    \leq & \E[R(\homega_i)-R(y^*)]\\
    \leq & O(G^2\eta_i),
\end{align*}
which implies that 
\begin{align*}
    \E[F(\homega_i)-F(y)]\leq & O(G^2\eta_i)-\frac{1}{\eta_i N_i}\E[\|\homega_i-\omega_{i-1}\|_2^2]+\frac{1}{\eta_i N_i}\E[\|y-\omega_{i-1}\|_2^2]\\
    \leq & O(G^2\eta_i)+\frac{1}{\eta_i N_i}\E[\|y-\omega_{i-1}\|_2^2].
\end{align*}
\end{proof}

Having these two lemmas, we can begin the proof.

\begin{proof}[Proof of Theorem~\ref{thm:local}]
The privacy guarantee comes directly from the basic composition theorem (See Theorem~\ref{thm:basic_com}).

Let $S_i=\{x_j\}_{N- N/2^{i-1} \leq j\leq N- N/2^{i}}$. 
Let $N_i=N/2^i,\epsilon_i=\epsilon/2^i,\delta_i=\delta/2^i$ and $\eta_i=\eta/2^{5i}$ where $\eta$ will be defined soon.
For $i\in[k]$, let $\HF_i(\omega)=\sum_{x_j\in S_i}f(\omega,x_j)+\frac{1}{\eta_iN_i}\|\omega-\omega_{i-1}\|_2^2$.

Let $\homega_0=\omega^*$, we have
\begin{align*}
    \E[F(\omega_k)]-F(\omega^*)=\sum_{i=1}^k\E[F(\homega_i)-F(\homega_{i-1})]+\E[F(\omega_k)-F(\homega_k)].
\end{align*}

First, Lemma~\ref{lm:accuray_phase} implies that
\begin{align*}
    \E[F(\omega_k)-F(\homega_k)]\leq & O(G\sqrt{\E[\|\omega_k-\homega_k\|_2^2]})\\
    \leq & O(\frac{G^2\eta_k\sqrt{ d\log(1/\delta_k)}}{\epsilon_k} + G^2\eta_k \sqrt{N_k})\\
    =& O(\frac{G^2\eta\sqrt{ d\log(N/\delta)}}{\epsilon N^3}+\frac{G^2\eta}{N^4}),
\end{align*}
which is negligible.

Then one has
\begin{align*}
    \sum_{i=1}^k\E[F(\homega_i)-F(\homega_{i-1})]\leq & \sum_{i=1}^k\frac{\E[\|\homega_{i-1}-\omega_{i-1}\|_2^2]}{\eta_iN_i}+O(G^2\eta_i)\\
    \leq & O(\frac{D^2}{\eta N}+ \eta G^2+\sum_{i=2}^{k}(\frac{G^2\eta_i d\log(1/\delta_i)}{\epsilon_i^2N_i}+G^2\eta_i))\\
    \leq & O(\frac{D^2}{\eta N}+\eta G^2+\frac{G^2\eta d\log(1/\delta)}{\epsilon^2N}).
\end{align*}

By setting $\eta=\frac{D}{G}\cdot\min\{\frac{1}{\sqrt{N}},\frac{\epsilon}{\sqrt{d\log(1/\delta)}}\}$, we get the excess population loss:
\begin{align*}
    \E[F(\homega_k)-F(\omega^*)]=O(GD(\frac{1}{\sqrt{N}}+\frac{\sqrt{d\log(1/\delta)}}{N\epsilon}).
\end{align*}
As for the gradient complexity, as we use $g(N_i,\epsilon_i,\delta_i)$ queries of gradients in $i$-th iteration, the total gradient complexity is $\sum_{i=1}^{k}g(N_i,\epsilon_i,\delta_i)$ as claimed.
\end{proof}
Note that Theorem~\ref{thm:local} allows the ERM algorithm has an extra $G^2/(\mu N)$ loss. This allows us to design a faster ERM algorithm compared Theorem~\ref{thm:strongly_SCO} by choosing a different set of parameters.

\begin{lemma}
\label{lm:DPERM_ite}
Under the assumption defined in Algorithm Private \ACSA, with
\begin{align*}
    O\big( N+
    \min\{\sqrt{\epsilon}N^{5/4}d^{1/8},\frac{\epsilon N^{\frac{3}{2}}}{d^{1/8}\log^{1/4}(1/\delta)}\} 
    \big)
\end{align*}
gradient complexity, one can get a solution $\omega_T$ such that
\begin{align*}
    \E[\HF(\omega_T)-\HF^*]=O(\frac{G^2}{\mu}(\frac{1}{N}+\frac{d\log(1/\delta)}{\epsilon^2 N^2})).
\end{align*}
\end{lemma}

\begin{proof}

By Lemma~\ref{lm:guarantee_ACSA}, one has
\begin{align*}
    \E[\HF(\omega_T)-\HF^*]=O\left(\frac{G^2/B+\sigma^2d/B^2}{\mu T}+\frac{GDd^{1/4}}{T}\right).
\end{align*}

Again, setting $\sigma=\frac{c_2GB\sqrt{T\log(1/\delta)}}{\epsilon N}$ one has
\begin{align*}
    \E[\HF(\omega_T)-\HF^{*}]=O\left(\frac{G^2}{\mu BT}+\frac{G^2d\log(1/\delta)}{\mu \epsilon^2 N^2}+\frac{GDd^{1/4}}{T}\right).
\end{align*}

Taking $T = 400 \lceil \min\{ N^{1/2}d^{1/4},\frac{N\epsilon}{d^{1/4}\sqrt{\log(1/\delta)}}\} \rceil$ and using $BT \geq N$ (which we will ensure), we have
\begin{align*}
    \E[\HF(\omega_T)-\HF^{*}]
    & =O\left(\frac{G^2}{\mu N}+\frac{G^2d\log(1/\delta)}{\mu \epsilon^2 N^2}+\frac{GD}{\sqrt{N}} +\frac{GD\sqrt{d\log(1/\delta)}}{\mu \epsilon N} \right)\\
    & =O\left(\frac{G^2}{\mu} \zeta + G D \sqrt{\zeta} \right).
\end{align*}
where $\zeta = \frac{1}{N} + \frac{d\log(1/\delta)}{\epsilon^2 N^2}$.

To ensure that Private $\ACSA$ is $(\epsilon,\delta)$-DP, we set $B=\lceil N/T+N\sqrt{\epsilon/T}\rceil =\Theta( N/T+N\sqrt{\epsilon/T})$. By our choice of $T$, we have $B \leq N/10$ and $\epsilon \leq c_1 B^2 T / N^2$. Hence, we can apply Lemma~\ref{lem:ACSA_DP} to conclude $(\epsilon,\delta)$-DP.

Hence, we have a $(\epsilon,\delta)$-DP for ERM with loss $O\left(\frac{G^2}{\mu} \zeta + G D \sqrt{\zeta} \right)$ with $\zeta = \frac{1}{N} + \frac{d\log(1/\delta)}{\epsilon^2 N^2}$. Note however that  Theorem~\ref{thm:local} requires us to have a DP-ERM algorithm with loss $O(\frac{G^2}{\mu} \zeta)$, namely, we have the extra term $O(G D \sqrt{\zeta})$. To remove this term, we follow the reduction in Lemma~\ref{lm:remove_dia}.

We note that the exact same proof as Lemma~\ref{lm:remove_dia} shows that for any $\zeta>0$, if we can solve strongly ERM with loss $O(G^2 \zeta^2/\mu + G D \zeta)$, then we can solve strongly ERM with loss $O(G^2 \zeta^2/\mu)$ by using the same number of gradient. This completes the proof.





\end{proof}
Before stating our result on SCO, we need the following variant of Lemma~\ref{lm:remove_dia}. The proof is essentially the same, we state it for future reference.

\begin{lemma}[Reduction to General Convex Case]
\label{lm:remove_dia2} Given $F$ is $G$-Lipschitz and $\mu$-strongly convex. 
Suppose for any $\epsilon,\delta<1/2$, we have an $(\epsilon,\delta)$-differentially private algorithm $\A$ which takes $\omega_0$ as the initial start point and $N$ samples i.i.d drawn from some distribution ${\cal P}$, and
outputs a solution $\omega_T$ such that
\begin{align*}
    \E[F(\omega_T)-F^*]=O\left(\frac{G^2}{\mu}(\frac{1}{N}+\frac{d\log(1/\delta)}{N^2\epsilon^2}) + GD(\frac{1}{\sqrt{N}}+\frac{\sqrt{d\log(1/\delta)}}{N\epsilon})\right),
\end{align*}
where $\omega^*=\arg\min_{\omega\in\K}F(\omega)$ and $D=\|\omega_0-\omega^*\|$.
Then by taking $\A$ as sub-procedure with some modifications on parameters, we can get an $(\epsilon,\delta)$-differentially private solution with excess population loss at most
\begin{align*}
\E[F(\omega_T)-F^*]=O\left(
\frac{G^2}{\mu}(\frac{1}{N}+\frac{d\log(1/\delta)}{N^2\epsilon^2})
\right).
\end{align*}
Furthermore, if $\A$ uses $g(N,\epsilon,\delta)$ many gradients, the new algorithm uses $\sum_{i\geq 1} g(N/2^i,\epsilon/2^i,\delta/2^i)$ many gradients.
\end{lemma}
\begin{proof}
The only difference to Lemma~\ref{lm:remove_dia} is that this algorithm takes $N/2^{k+1-i}$ samples instead of $N$ samples in the $i$-th step for $k=\lceil \log\log N^3\rceil$, so it may have less gradient complexity. 
The rest of the proof is identical.
\end{proof}

Now, we can get the result for general convex case by Theorem~\ref{thm:local} and Lemma~\ref{lm:DPERM_ite}, then extend it to strongly convex case by Lemma~\ref{lm:remove_dia2}.

\begin{theorem}[DP-SCO, Theorem~\ref{thm:DP-SCO} restated]
Suppose $\epsilon,\delta\leq 1/2$.
Let $\{f(\cdot,x)\}_{x\in \Xi}$ is convex and $G$-Lipschitz with respect to $\ell_2$ norm and convex over $\K_r$, where $r=\frac{D\sqrt{d\log(1/\delta)}}{\epsilon N}$, there is an $(\epsilon,\delta)$-differentially private algorithm which takes
\begin{align*}
    O(N+\min\{\sqrt{\epsilon}N^{5/4}d^{1/8},\frac{\epsilon N^{3/2}}{d^{1/8}\log^{1/4}(1/\delta)}\})
\end{align*}
gradient queries to get a solution $\omega_T$
\begin{align*}
    \E[F(\omega_T)-F(\omega^*)]=O(GD(\frac{1}{\sqrt{N}}+\frac{\sqrt{d\log(1/\delta)}}{N\epsilon}).
\end{align*}
Moreover, if $\{f(\cdot,x)\}_{x\in \Xi}$ is also $\mu$-strongly convex over $\K_r$, we can use the same gradient complexity and get a solution $\omega$ such that:
\begin{align*}
    \E[F(\omega_T)-F(\omega^*)]=O\left(\frac{G^2}{\mu}(\frac{d\log(1/\delta)}{\epsilon^2N^2}+\frac{1}{N})\right).
\end{align*}
\end{theorem}

\newpage

\addcontentsline{toc}{section}{References}
\bibliographystyle{alpha}
\bibliography{ref}

\newcommand{\etalchar}[1]{$^{#1}$}
\begin{thebibliography}{WLK{\etalchar{+}}17}

\bibitem[ACG{\etalchar{+}}16]{acg+16}
Martin Abadi, Andy Chu, Ian Goodfellow, H~Brendan McMahan, Ilya Mironov, Kunal
  Talwar, and Li~Zhang.
\newblock Deep learning with differential privacy.
\newblock In {\em Proceedings of the 2016 ACM SIGSAC Conference on Computer and
  Communications Security}, pages 308--318, 2016.

\bibitem[AFKT21]{afkt21}
Hilal Asi, Vitaly Feldman, Tomer Koren, and Kunal Talwar.
\newblock Private stochastic convex optimization: Optimal rates in $\ell_1$
  geometry.
\newblock {\em arXiv preprint arXiv:2103.01516}, 2021.

\bibitem[BDRS18]{bdrs18}
Mark Bun, Cynthia Dwork, Guy~N Rothblum, and Thomas Steinke.
\newblock Composable and versatile privacy via truncated cdp.
\newblock In {\em Proceedings of the 50th Annual ACM SIGACT Symposium on Theory
  of Computing}, pages 74--86, 2018.

\bibitem[BFGT20]{bfgt20}
Raef Bassily, Vitaly Feldman, Crist{\'o}bal Guzm{\'a}n, and Kunal Talwar.
\newblock Stability of stochastic gradient descent on nonsmooth convex losses.
\newblock {\em arXiv preprint arXiv:2006.06914}, 2020.

\bibitem[BFTT19]{bftt19}
Raef Bassily, Vitaly Feldman, Kunal Talwar, and Abhradeep~Guha Thakurta.
\newblock Private stochastic convex optimization with optimal rates.
\newblock In {\em Advances in Neural Information Processing Systems}, pages
  11282--11291, 2019.

\bibitem[BGN21]{bgn21}
Raef Bassily, Crist{\'o}bal Guzm{\'a}n, and Anupama Nandi.
\newblock Non-euclidean differentially private stochastic convex optimization.
\newblock {\em arXiv preprint arXiv:2103.01278}, 2021.

\bibitem[BST14]{bst14}
Raef Bassily, Adam Smith, and Abhradeep Thakurta.
\newblock Private empirical risk minimization: Efficient algorithms and tight
  error bounds.
\newblock In {\em 2014 IEEE 55th Annual Symposium on Foundations of Computer
  Science}, pages 464--473. IEEE, 2014.

\bibitem[CBCG04]{ccg04}
Nicolo Cesa-Bianchi, Alex Conconi, and Claudio Gentile.
\newblock On the generalization ability of on-line learning algorithms.
\newblock {\em IEEE Transactions on Information Theory}, 50(9):2050--2057,
  2004.

\bibitem[CM08]{cm08}
Kamalika Chaudhuri and Claire Monteleoni.
\newblock Privacy-preserving logistic regression.
\newblock In {\em NIPS}, volume~8, pages 289--296. Citeseer, 2008.

\bibitem[CMS11]{cms11}
Kamalika Chaudhuri, Claire Monteleoni, and Anand~D Sarwate.
\newblock Differentially private empirical risk minimization.
\newblock {\em Journal of Machine Learning Research}, 12(3), 2011.

\bibitem[DBW12]{dbw12}
John~C Duchi, Peter~L Bartlett, and Martin~J Wainwright.
\newblock Randomized smoothing for stochastic optimization.
\newblock {\em SIAM Journal on Optimization}, 22(2):674--701, 2012.

\bibitem[DMNS06]{dmns06}
Cynthia Dwork, Frank McSherry, Kobbi Nissim, and Adam Smith.
\newblock Calibrating noise to sensitivity in private data analysis.
\newblock In {\em Theory of cryptography conference}, pages 265--284. Springer,
  2006.

\bibitem[DR14]{dr14}
Cynthia Dwork and Aaron Roth.
\newblock The algorithmic foundations of differential privacy.
\newblock {\em Foundations and Trends in Theoretical Computer Science},
  9(3-4):211--407, 2014.

\bibitem[Fel16]{fel16}
Vitaly Feldman.
\newblock Generalization of erm in stochastic convex optimization: The
  dimension strikes back.
\newblock {\em arXiv preprint arXiv:1608.04414}, 2016.

\bibitem[FKT20]{fkt20}
Vitaly Feldman, Tomer Koren, and Kunal Talwar.
\newblock Private stochastic convex optimization: optimal rates in linear time.
\newblock In {\em Proceedings of the 52nd Annual ACM SIGACT Symposium on Theory
  of Computing}, pages 439--449, 2020.

\bibitem[FMTT18]{fmtt18}
Vitaly Feldman, Ilya Mironov, Kunal Talwar, and Abhradeep Thakurta.
\newblock Privacy amplification by iteration.
\newblock In {\em 2018 IEEE 59th Annual Symposium on Foundations of Computer
  Science (FOCS)}, pages 521--532. IEEE, 2018.

\bibitem[FTS17]{fts17}
Kazuto Fukuchi, Quang~Khai Tran, and Jun Sakuma.
\newblock Differentially private empirical risk minimization with input
  perturbation.
\newblock In {\em International Conference on Discovery Science}, pages 82--90.
  Springer, 2017.

\bibitem[GL12]{gl12}
Saeed Ghadimi and Guanghui Lan.
\newblock Optimal stochastic approximation algorithms for strongly convex
  stochastic composite optimization i: A generic algorithmic framework.
\newblock {\em SIAM Journal on Optimization}, 22(4):1469--1492, 2012.

\bibitem[HRS16]{hrs16}
Moritz Hardt, Ben Recht, and Yoram Singer.
\newblock Train faster, generalize better: Stability of stochastic gradient
  descent.
\newblock In {\em International Conference on Machine Learning}, pages
  1225--1234. PMLR, 2016.

\bibitem[INS{\etalchar{+}}19]{ins+19}
Roger Iyengar, Joseph~P Near, Dawn Song, Om~Thakkar, Abhradeep Thakurta, and
  Lun Wang.
\newblock Towards practical differentially private convex optimization.
\newblock In {\em 2019 IEEE Symposium on Security and Privacy (SP)}, pages
  299--316. IEEE, 2019.

\bibitem[JNN19]{jnn19}
Prateek Jain, Dheeraj Nagaraj, and Praneeth Netrapalli.
\newblock Making the last iterate of sgd information theoretically optimal.
\newblock In {\em Conference on Learning Theory}, pages 1752--1755. PMLR, 2019.

\bibitem[JT14]{jt14}
Prateek Jain and Abhradeep~Guha Thakurta.
\newblock (near) dimension independent risk bounds for differentially private
  learning.
\newblock In {\em International Conference on Machine Learning}, pages
  476--484. PMLR, 2014.

\bibitem[KJ16]{kj16}
Shiva~Prasad Kasiviswanathan and Hongxia Jin.
\newblock Efficient private empirical risk minimization for high-dimensional
  learning.
\newblock In {\em International Conference on Machine Learning}, pages
  488--497. PMLR, 2016.

\bibitem[KST12]{kst12}
Daniel Kifer, Adam Smith, and Abhradeep Thakurta.
\newblock Private convex empirical risk minimization and high-dimensional
  regression.
\newblock In {\em Conference on Learning Theory}, pages 25--1. JMLR Workshop
  and Conference Proceedings, 2012.

\bibitem[Mir17]{mir18}
Ilya Mironov.
\newblock R{\'e}nyi differential privacy.
\newblock In {\em 2017 IEEE 30th Computer Security Foundations Symposium
  (CSF)}, pages 263--275. IEEE, 2017.

\bibitem[Nes05]{nes05}
Yu~Nesterov.
\newblock Smooth minimization of non-smooth functions.
\newblock {\em Mathematical programming}, 103(1):127--152, 2005.

\bibitem[RBHT09]{rbht09}
Benjamin~IP Rubinstein, Peter~L Bartlett, Ling Huang, and Nina Taft.
\newblock Learning in a large function space: Privacy-preserving mechanisms for
  svm learning.
\newblock {\em arXiv preprint arXiv:0911.5708}, 2009.

\bibitem[R{\'e}n61]{ren61}
Alfr{\'e}d R{\'e}nyi.
\newblock On measures of entropy and information.
\newblock In {\em Proceedings of the Fourth Berkeley Symposium on Mathematical
  Statistics and Probability, Volume 1: Contributions to the Theory of
  Statistics}. The Regents of the University of California, 1961.

\bibitem[SSSSS09]{sssss09}
Shai Shalev-Shwartz, Ohad Shamir, Nathan Srebro, and Karthik Sridharan.
\newblock Stochastic convex optimization.
\newblock In {\em COLT}, 2009.

\bibitem[TTZ14]{ttz14}
Kunal Talwar, Abhradeep Thakurta, and Li~Zhang.
\newblock Private empirical risk minimization beyond the worst case: The effect
  of the constraint set geometry.
\newblock {\em arXiv preprint arXiv:1411.5417}, 2014.

\bibitem[TTZ15]{ttz15}
Kunal Talwar, Abhradeep Thakurta, and Li~Zhang.
\newblock Nearly-optimal private lasso.
\newblock In {\em Proceedings of the 28th International Conference on Neural
  Information Processing Systems-Volume 2}, pages 3025--3033, 2015.

\bibitem[WLK{\etalchar{+}}17]{wlk+17}
Xi~Wu, Fengan Li, Arun Kumar, Kamalika Chaudhuri, Somesh Jha, and Jeffrey
  Naughton.
\newblock Bolt-on differential privacy for scalable stochastic gradient
  descent-based analytics.
\newblock In {\em Proceedings of the 2017 ACM International Conference on
  Management of Data}, pages 1307--1322, 2017.

\bibitem[WYX17]{wyx17}
Di~Wang, Minwei Ye, and Jinhui Xu.
\newblock Differentially private empirical risk minimization revisited: Faster
  and more general.
\newblock In {\em Advances in Neural Information Processing Systems}, pages
  2722--2731, 2017.

\bibitem[YNS12]{yns11}
Farzad Yousefian, Angelia Nedi{\'c}, and Uday~V Shanbhag.
\newblock On stochastic gradient and subgradient methods with adaptive
  steplength sequences.
\newblock {\em Automatica}, 48(1):56--67, 2012.

\bibitem[ZZMW17]{zzmw17}
Jiaqi Zhang, Kai Zheng, Wenlong Mou, and Liwei Wang.
\newblock Efficient private erm for smooth objectives.
\newblock {\em arXiv preprint arXiv:1703.09947}, 2017.

\end{thebibliography}
\newpage

\appendix
\section{Proof of Theorem~\ref{thm:dp_meta}}
\label{sec:appendix_a}
As mentioned before, we can use the result in \cite{bdrs18} to give a formal proof of our result.
Before we start, let us define something necessary.

\begin{definition}[Truncated CDP]
Let $\rho>0$ and $\omega>1$.
A randomized algorithm $\M:\mathbb{N}^{|\Xi|}\rightarrow R$ satisfies $\omega$-truncated $\rho$-concentrated differential privacy (or $(\rho,\omega)$-tCDP) if for all neighboring $S,S'$ that differ in a single entry,
\begin{align*}
    \forall \alpha \in(1, \omega), \mathrm{D}_{\alpha}\left(\M(S) \| \M\left(S^{\prime}\right)\right) \leq \rho \alpha,
\end{align*}
where $D_{\alpha}(\cdot\|\cdot)$ denotes the Rényi divergence \cite{ren61} of order $\alpha$ (in nats, rather than bits).
\end{definition}

Similar to classic differential privacy, tCDP also enjoys a property of composition:

\begin{lemma}[Composition of tCDP]
\label{lm:com_tCDP}
Let $\M_1:\mathbb{N}^{|\Xi|}\rightarrow R_1$ satisfy $(\rho,\omega)$-tCDP and let $\M_2:\mathbb{N}^{|\Xi|}\times R_1\rightarrow R_2$ satisfy $(\rho',\omega')$-tCDP for all $y\in R_1$.
Difine $\M:\mathbb{N}^{|\Xi|}\rightarrow R_3$ by $\M(S)=\M_2(S,\M_1(S))$.
Then $\M$ satisfies $(\rho+\rho',\min\{\omega,\omega'\})$-tCDP.
\end{lemma}

Now we state the main result of \cite{bdrs18}:
\begin{theorem}[Privacy Amplification By Subsampling]
\label{thm:tCDP}
Let $\rho,s\in(0,0.1]$ and $B,N\in \mathbb{N}$ with $q=B/N$ and $\log(1/q)\geq3\rho(2+\log_2(1/\rho))$.
Let $\M:\mathbb{N}^{|\Xi|}\rightarrow R$ satisfy $(\rho,\omega')$-tCDP for $\omega'\geq\frac{\log(1/q)}{2\rho}\geq 3$.
Define the mechanism $\M_q:\mathbb{N}^{|\Xi|}\rightarrow R$ by $\M_q(S)=\M(S_q)$ where $S_q\in \mathbb{N}^{|\Xi|}$ is the restriction of $S$ to the entries specified by a uniformly ransom subset of size $B$.

The algorithm $\M_q$ satisfies $(13q^2\rho,\omega)$-tCDP for
\begin{align*}
    \omega=\frac{\log(1/q)}{4\rho}.
\end{align*}
\end{theorem}

This theorem can apply to our algorithm $\mathsf{META_{DP}}$ directly, as we are using subsampling without replacement.
More specifically, we are using subsampling Gaussian Mechanism, and for Gaussian Mechanism we have the following fact:
\begin{fact}
Let $P=\N(1,1/2\rho)$ and $Q=\N(0,1/2\rho)$. Then $D_{\alpha}(P\mid Q)=\rho \alpha$ for all $\alpha\in (1,\infty)$.
In other word, the Gaussian Mechanism with sensitive 1 satisfies $(\rho,\infty)$-tCDP.
\end{fact}

Now we can start our proof.
\begin{proof}[Proof of Theorem~\ref{thm:dp_meta}]
For the $t$-th phase of $\mathsf{META_{DP}}$, let $\M(S)=\sum_{x\in S}\nabla f(\omega_{t-1},x)+v$ where $v\sim\N(0,\sigma^2I_{d\times d})$. 
As we are considering $G$-Lipschitz function $f$, then we know that $\|\nabla (f)\|_2\leq G$, which means that $\M$ is $(\rho,\infty)$-tCDP where $\rho=G^2/(2\sigma^2)$.

Assume our parameters satisfy the precondition of Theorem~\ref{thm:tCDP} first, then we know that the $t$-th phase of $\mathsf{META_{DP}}$ is $(13q^2\rho,1/\rho)$-tCDP.
By the composition property (Lemma~\ref{lm:com_tCDP}), we know that $\mathsf{META_{DP}}$ is $(13Tq^2\rho,1/\rho)$-tCDP.

When $Tq^2\rho\cdot \frac{\log(1/\delta)}{\epsilon}\leq O( \epsilon)$ and $\frac{\log(1/\delta)}{\epsilon}\leq O(\frac{1}{\rho})$, we know that $\mathsf{META_{DP}}$ is $(\epsilon,\delta)$-differentially private \cite{bdrs18}.

By setting $\sigma=\frac{c_2GB\sqrt{T\log(1/\delta)}}{\epsilon N}$, we have that $\rho = \frac{\epsilon^2 N^2}{2 c_2^2 B^2 T \log(1/\delta)}$. Together with the assumption $\epsilon\leq c_1B^2T/N^2$, we have both
$Tq^2\rho\cdot \frac{\log(1/\delta)}{\epsilon}\leq O( \epsilon)$ and $\frac{\log(1/\delta)}{\epsilon}\leq O(\frac{1}{\rho})$ as claimed. This completes the proof.

\end{proof}

\end{document}